%% file: flexgen.tex
\newif\ifisarxiv
\isarxivfalse

\documentclass{article}
\usepackage{wrapfig}

\usepackage{microtype}
\usepackage{graphicx}
\usepackage{subfigure}
\usepackage{booktabs} %
\usepackage{thmtools}
\usepackage{thm-restate}
\usepackage{mdframed}
\usepackage{tabularx}

\newenvironment{boxedalgorithmic}
  {\mdframed[topline=true,bottomline=true,rightline=false,leftline=false,
      linecolor=black,linewidth=0.6pt,backgroundcolor=white,roundcorner=0.6pt]
   \begin{algorithmic}}
  {\end{algorithmic}\endmdframed}

\usepackage{hyperref}

\ifisarxiv
\usepackage[accepted]{arxiv}
\newlength{\mywidth}
\else
\usepackage[accepted]{icml2023}
\newlength{\mywidth}
\fi

\usepackage{amsmath}
\usepackage{amssymb}
\usepackage{mathtools}
\usepackage{amsthm}

\usepackage[capitalize]{cleveref}

\theoremstyle{plain}
\newtheorem{theorem}{Theorem}[section]

\newtheorem{lemma}[theorem]{Lemma}

\theoremstyle{definition}
\newtheorem{definition}[theorem]{Definition}

\theoremstyle{remark}

\usepackage{ifthen}
\usepackage{xspace}
\usepackage{multirow}
\usepackage{multicol}
\usepackage{enumitem}
\setitemize{noitemsep,topsep=0pt,parsep=0pt,partopsep=0pt}

\newcommand{\sys}{FlexGen\xspace}
\newcommand{\showcomments}{yes}

\newcommand{\x}{\mathbf{x}}

\newcommand{\wq}{\mathbf{w}_Q}

\newcommand{\wv}{\mathbf{w}_V}

\newcommand{\wk}{\mathbf{w}_K}

\newcommand{\wo}{\mathbf{w}_O}

\newcommand{\wa}{\mathbf{w}_1}

\newcommand{\wb}{\mathbf{w}_2}

\newcommand{\xq}{\mathbf{x}_Q}

\newcommand{\xv}{\mathbf{x}_V}

\newcommand{\xk}{\mathbf{x}_K}

\newcommand{\xo}{\mathbf{x}_{\textsf{Out}}}

\newcommand{\tx}{\mathbf{t}}
\newcommand{\tq}{\mathbf{t}_Q}

\newcommand{\txo}{\mathbf{t}_{\textsf{Out}}}

\newcommand\todo[1]{\ifthenelse{\equal{\showcomments}{yes}}{{\color{red} TODO: #1}}{\ignorespaces}}
\newcommand\dan[1]{\ifthenelse{\equal{\showcomments}{yes}}{{\color{red} DF: #1}}{\ignorespaces}}

\newif\ifcomments
\commentstrue 
\ifcomments
    \newcommand{\ying}[1]{{\color{green}{\bf\sf [Ying: #1]}}}
    \newcommand{\lianmin}[1]{{\color{red}{\bf\sf [Lianmin: #1]}}}
    \newcommand{\ce}[1]{{\color{cyan}{\bf\sf [Ce: #1]}}}
    \newcommand{\pl}[1]{{\color{cyan}{\bf\sf [Percy: #1]}}}
    \newcommand{\chris}[1]{{\color{blue}{\bf\sf [Chris: #1]}}}
    \newcommand{\binhang}[1]{{\color{blue}{\bf\sf [Binhang: #1]}}}
    \newcommand{\beidi}[1]{{\color{cyan}{\bf\sf [Beidi: #1]}}}
    \newcommand{\clark}[1]{{\color{red}{\bf\sf [Clark: #1]}}}
    \newcommand{\ion}[1]{{\color{red}{\bf\sf [Ion: #1]}}}
    \newcommand{\fixme}[1]{{\color{orange}{#1}}}
\else
    \newcommand{\ying}[1]{}
    \newcommand{\lianmin}[1]{}
    \newcommand{\ce}[1]{}
    \newcommand{\pl}[1]{{}
    \newcommand{\chris}[1]{}
    \newcommand{\binhang}[1]{}
    \newcommand{\beidi}[1]{}
    \newcommand{\clark}[1]{}
    \newcommand{\ion}[1]{}
    \newcommand{\fixme}[1]{#1}
\fi

\renewcommand{\algorithmicrequire}{\textbf{Input:}}
\renewcommand{\algorithmicensure}{\textbf{Output:}}

\icmltitlerunning{
~\hfill
FlexGen: High-Throughput Generative Inference of Large Language Models with a Single GPU~
\hfill
}

\begin{document}

\ifisarxiv
\else
\twocolumn[
\icmltitle{
FlexGen: High-Throughput Generative Inference of Large Language Models \\ with a Single GPU
}
\icmlsetsymbol{equal}{*}

\begin{icmlauthorlist}
\icmlauthor{Ying Sheng}{stanford}
\icmlauthor{Lianmin Zheng}{berkeley}
\icmlauthor{Binhang Yuan}{eth}
\icmlauthor{Zhuohan Li}{berkeley}
\icmlauthor{Max Ryabinin}{yandex,hse}
\icmlauthor{Daniel Y. Fu}{stanford}
\icmlauthor{Zhiqiang Xie}{stanford}
\\
\icmlauthor{Beidi Chen}{meta,cmu}
\icmlauthor{Clark Barrett}{stanford}
\icmlauthor{Joseph E. Gonzalez}{berkeley}
\icmlauthor{Percy Liang}{stanford}
\icmlauthor{Christopher R\'e}{stanford}
\icmlauthor{Ion Stoica}{berkeley}
\icmlauthor{Ce Zhang}{eth}
\end{icmlauthorlist}

\icmlaffiliation{stanford}{Stanford University}
\icmlaffiliation{berkeley}{UC Berkeley}
\icmlaffiliation{eth}{ETH Zurich}
\icmlaffiliation{meta}{Meta}
\icmlaffiliation{yandex}{Yandex}
\icmlaffiliation{cmu}{Carnegie Mellon University}
\icmlaffiliation{hse}{HSE University}

\icmlcorrespondingauthor{Ying Sheng}{ying1123@stanford.edu}

\icmlkeywords{FlexGen: large language models, memory optimizations, offloading, compression, generative pre-trained transformers}

\vskip 0.3in
]
\fi

\printAffiliationsAndNotice{}  %

\input{sec.intro.tex}

\input{sec.related_work.tex}

\input{sec.background.tex}

\input{sec.offload.tex}

\input{sec.approximate.tex}
\input{sec.evaluation.tex}

\input{sec.conclusion.tex}

\bibliography{flexgen}
\bibliographystyle{icml2023}

\clearpage
\appendix
\input{sec.appendix.tex}
\end{document}

%% file: sec.intro.tex
\begin{abstract}
The high computational and memory requirements of large language model (LLM) inference make it feasible only with multiple high-end accelerators.
Motivated by the emerging demand for latency-insensitive tasks with batched processing, this paper initiates the study of high-throughput LLM inference using limited resources, such as a single commodity GPU.
We present \sys, a high-throughput generation engine for running LLMs with limited GPU memory.
\sys can be flexibly configured under various hardware resource constraints by aggregating memory and computation from the GPU, CPU, and disk. By solving a linear programming problem, it searches for efficient patterns to store and access tensors.
\sys further compresses the weights and the attention cache to 4 bits with negligible accuracy loss.
These techniques enable \sys to have a larger space of batch size choices and thus significantly increase maximum throughput.
As a result, when running OPT-175B on a single 16GB GPU, \sys achieves significantly higher throughput compared to state-of-the-art offloading systems, reaching a generation throughput of 1 token/s for the first time with an effective batch size of $144$. 
On the HELM benchmark, \sys can benchmark a 30B model with a 16GB GPU on 7 representative sub-scenarios in 21 hours.
The code is available at \url{https://github.com/FMInference/FlexGen}.
\end{abstract}

\section{Introduction}
In recent years, large language models (LLMs) have demonstrated strong performance across a wide range of tasks~\cite{brown2020language,bommasani2021opportunities,zhang2022opt,chowdhery2022palm}.
Along with these unprecedented capabilities, generative LLM inference comes with unique challenges.
These models can have billions, if not trillions of parameters~\cite{chowdhery2022palm,fedus2022switch}, which leads to extremely high computational and memory requirements to run. For example, GPT-175B requires $325$GB of GPU memory simply to load its model weights. Fitting this model onto GPUs would require at least five A100 (80GB) GPUs and complex parallelism strategies~\cite{pope2022efficiently, aminabadi2022deepspeed}. Thus, lowering LLM inference resource requirements has recently attracted intense interest.

\begin{figure}[t]
	\centering
\includegraphics[width=0.91\mywidth]{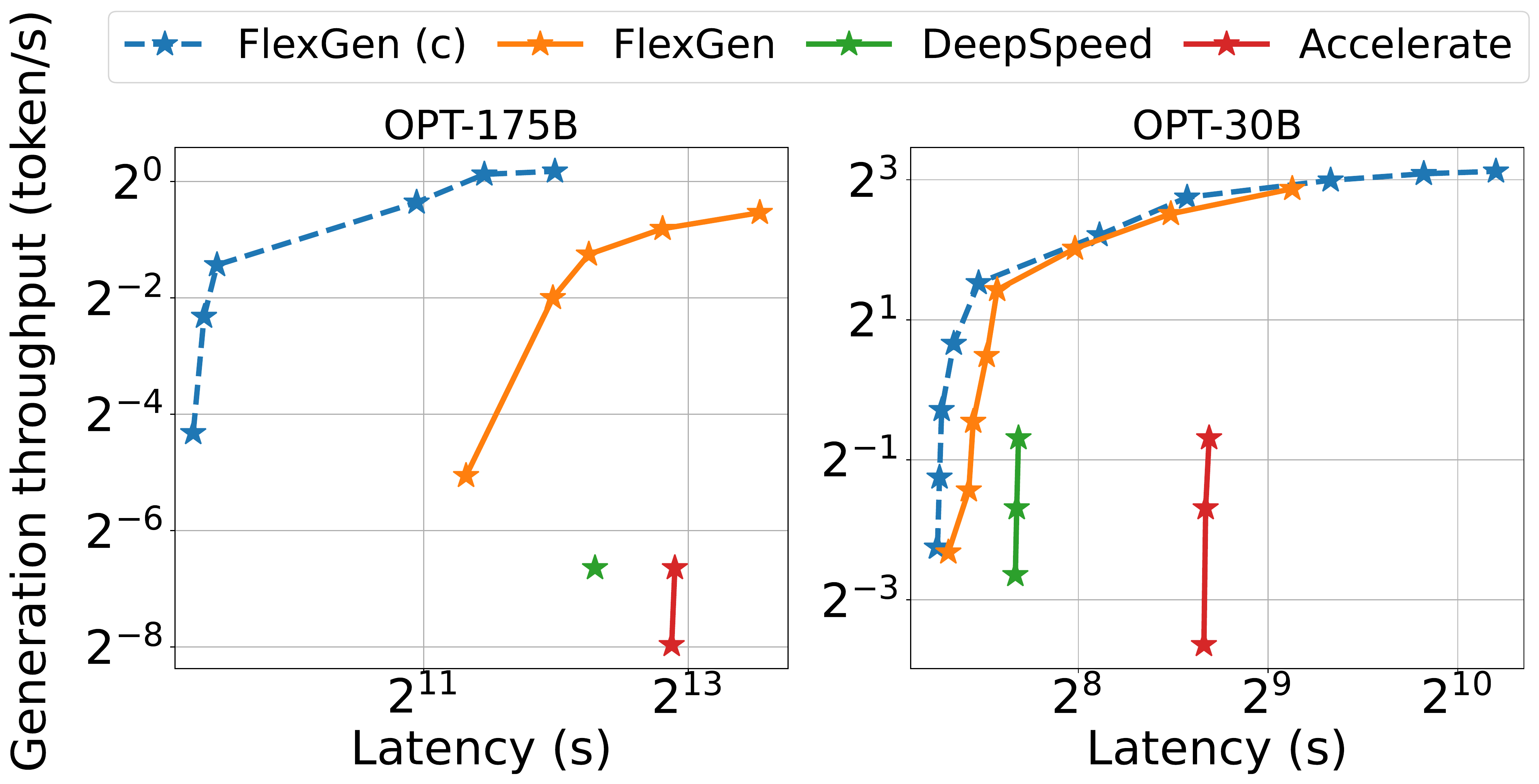}
\vspace{-0.8em}
\caption{The total latency for a block and throughput trade-offs of three offloading-based systems for OPT-175B (left) and OPT-30B (right) on a single NVIDIA T4 (16 GB) GPU with 208 GB CPU DRAM and 1.5TB SSD. \sys achieves a new Pareto-optimal frontier with $100\times$ higher maximum throughput for OPT-175B. Other systems cannot further increase throughput due to out-of-memory issues. ``(c)'' denotes compression.}
\vspace{-1.5em}
\label{fig:throughput_vs_latency}
\end{figure}

In this paper, we focus on a  
setting that we call \textit{throughput-oriented generative inference}.
In addition to interactive use cases such as chatbots, LLMs are also applied to 
many ``back-of-house'' tasks such as 
benchmarking~\cite{liang2022holistic}, information extraction~\cite{xsum}, data wrangling~\cite{narayan2022can}, and form processing~\cite{chen2021spreadsheetcoder}. One key characteristic of these tasks is that they often require running LLM inference in 
batches over a large number of tokens (e.g., all the documents in a company's corpus),
and are less sensitive to latency.
As a result, it is
possible to trade off latency for higher throughput in these workloads, providing opportunities to 
reduce resource requirements.

Prior efforts to lower resource requirements of LLM inference correspond to three directions:
(1) \textit{model compression} to decrease total memory footprint~\cite{dettmers2022gptint,yao2022zeroquant,frantar2022gptq,xiao2022smoothquant};
(2) \textit{collaborative inference} to amortize inference cost via decentralization~\cite{borzunov2022petals};
and (3) \textit{offloading} to utilize memory from CPU and disk~\cite{aminabadi2022deepspeed, huggingfaceAccelerate}.
These techniques have significantly lowered the resource requirements for using LLMs, but there are distinct limitations.
Research in the first two directions often assume that the model
fits into the GPU memory and thereby struggle to run 175B-scale models with a single commodity GPU. On the other hand, state-of-the-art offloading-based systems in the third category do not achieve acceptable throughput on a single GPU due to inefficient I/O scheduling and tensor placement. For example, these systems can be bottlenecked by small batch sizes (e.g., batch sizes of only one or two for OPT-175B in some cases). 

\begin{wrapfigure}{r}{0.17\textwidth}
\includegraphics[width=0.17\textwidth]{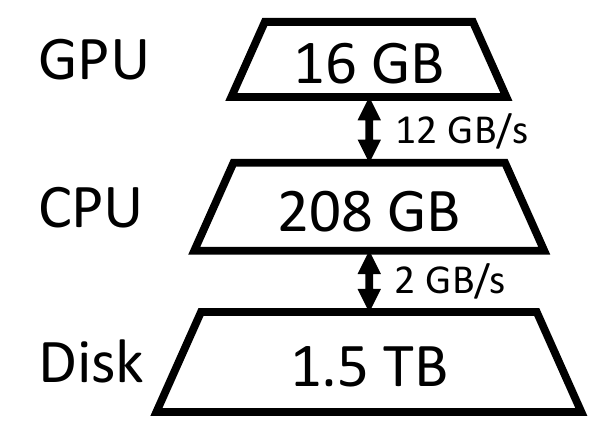}
\vspace{-2em}
\end{wrapfigure}

Our focus is designing efficient \textit{offloading} strategies for high-throughput generative inference, \textit{on a single commodity GPU}. To run an LLM with limited GPU memory, we can offload it to secondary storage and perform computation part-by-part by partially loading it. On a typical machine, there are three levels of the memory hierarchy, as illustrated in the figure to the right. Higher levels are faster but scarce, while lower levels are slower but abundant.
In throughput-oriented scenarios, we can sacrifice latency by using a large batch size, and amortize the expensive I/O operations among different memory hierarchies over a large batch of inputs, overlapped with computation.
\cref{fig:throughput_vs_latency} shows the latency-throughput trade-off of three inference systems with offloading on a single NVIDIA T4 (16 GB) GPU. Note that the performance in terms of latency and throughput on limited resources is significantly inferior to that of the cases with sufficient resources.

Achieving high-throughput generative inference with limited GPU memory is challenging even if we can sacrifice the latency.
The first challenge is to design an \textit{efficient offloading strategy}.
During generative inference, there are three kinds of tensors: weights, activations, and key-value (KV) cache. The strategy should specify what tensors to offload, where to offload them within the three-level memory hierarchy, and when to offload them during inference.
The batch-by-batch, token-by-token, and layer-by-layer structure of the computation forms a complex dependency graph where there are multiple ways to conduct computation. Together, these choices form a complex design space.
Existing offloading-based inference systems~\cite{aminabadi2022deepspeed, huggingfaceAccelerate} inherit strategies from training, which turn out to be some suboptimal points for inference, performing excessive I/O and achieving throughput far below theoretical hardware limits.

The second challenge is to develop \textit{effective compression strategies}. Previous works have demonstrated promising results in compressing the weights and activations of LLMs. However, when combining compression with offloading for high-throughput inference, the I/O costs and memory reduction of the weights and KV cache become more important, motivating alternative compression schemes.

To address these challenges, we present \sys, an offloading framework for high-throughput LLM inference.
\sys aggregates memory from the GPU, CPU, and disk, and efficiently schedules I/O operations, along with possible compression methods and distributed pipeline parallelism.

{\bf \underline{(Contribution 1)}}
We formally define a search space of possible offloading strategies by considering computation schedule, tensor placement, and computation delegation.
We prove that our search space captures a computation order with I/O complexity within $2\times$ of optimality.
We then develop a linear programming-based search algorithm to optimize the throughput within the search space.
This algorithm can be configured for various hardware specifications and can be easily extended to incorporate latency and throughput constraints, thus helping to navigate the trade-off space smoothly.
Compared with existing strategies, our solution unifies the placement of weights, activations, and the KV cache, enabling a dramatically higher batch size upper bound, which is key to achieving high throughput.

{\bf \underline{(Contribution 2)}}
We show that it is possible to compress both the weights and KV cache for LLMs like OPT-175B to 4 bits without retraining or calibration, all with negligible accuracy loss. This is achieved through fine-grained group-wise quantization~\cite{shen2020q}, which is suitable for reducing I/O costs and memory usage during offloading.

{\bf \underline{(Contribution 3)}}
We demonstrate the efficiency of \sys by running OPT-175B on NVIDIA T4 (16GB) GPUs. 
Compared to DeepSpeed Zero-Inference~\cite{aminabadi2022deepspeed} and Hugging Face Accelerate~\cite{huggingfaceAccelerate}, two state-of-the-art offloading-based inference systems, \sys often allows a batch size that is orders of magnitude larger.
As a result, \sys can achieve much higher throughputs. On a single T4 GPU with 208 GB CPU DRAM and 1.5 TB SSD, input sequence length 512, and output sequence length 32:
\begin{itemize}
\item With the same latency of $5000$ seconds, \sys (effective batch size 64, or 2048 tokens in total) can achieve more than $40\times$ higher throughput than DeepSpeed Zero-Inference (batch size 1, or 32 tokens in total), while Hugging Face Accelerate cannot complete a single batch.
\item By allowing a higher latency of $12000$ seconds, FlexGen achieves $69\times$
higher maximum throughput compared to baselines because it can enlarge the effective batch size to $256$ (8192 tokens generated in total), while DeepSpeed Zero-Inference and Hugging Face Accelerate cannot use a batch size larger than 2 due to out-of-memory issues.
\item If allowing 4-bit compression, \sys can reach $100\times$ higher maximum throughput with effective batch size 144 (4608 tokens generated in total) with latency 4000 seconds by holding all weights in CPU and getting rid of disk offloading.
\end{itemize}

We also compare offloading and decentralized collective inference based on \sys and Petals~\cite{borzunov2022petals} as two representative systems. We conduct comparisons between the two systems from the aspects of delay and bandwidth of the decentralized network and output sequence length.
The results show that \sys outperforms a decentralized Petals cluster in terms of per-GPU throughput and can even achieve lower latency in certain cases.

%% file: sec.related_work.tex
\section{Related Work}
\label{sec:related_work}

Given the recent advances of LLMs, LLM inference has become an important workload, encouraging active research from both the \textbf{system} side and the \textbf{algorithm} side. 

Recent years have witnessed the emergence of systems specialized for LLM inference, such as FasterTransformer~\cite{nvidiaft},  Orca~\cite{yu2022orca}, LightSeq~\cite{wang2021lightseq}, PaLM inference~\cite{pope2022efficiently}, TurboTransformers~\cite{fang2021turbotransformers}, DeepSpeed Inference~\cite{aminabadi2022deepspeed}, and Hugging Face Accelerate~\cite{huggingfaceAccelerate}. Unfortunately, most of these systems focus on latency-oriented scenarios with high-end accelerators, limiting their deployment for throughput-oriented inference on easily accessible hardware.
To enable LLM inference on such commodity hardware, offloading is an essential technique --- as far as we know, among current systems, only DeepSpeed Zero-Inference and Hugging Face Accelerate support offloading.
These inference systems typically inherit the offloading techniques from training systems~\cite{rajbhandari2021zero, ren2021zero,li2022harmony,huang2020swapadvisor,wang2018superneurons} but ignore the special computational property of generative inference.
They fail to exploit the structure of the throughput-oriented LLM inference computation and miss great opportunities for efficient scheduling of I/O traffic. Another attempt to enable LLM inference on accessible hardware is collaborative computing proposed by Petals~\cite{borzunov2022petals}.

There are also many algorithm-oriented works that relax certain aspects of computation in LLM inference to accelerate the computation or reduce the memory footprint.
Both sparsification~\cite{hoefler2021sparsity, frantar2023massive} and quantization~\cite{kwon2022alphatuning,yao2022zeroquant,park2022nuqmm,xiao2022smoothquant,frantar2022gptq,dettmers2022gptint} have been adopted for LLM inference.
On the quantization side, prior works have shown weights can be compressed down to 3 bits without compressing activations~\cite{frantar2022gptq}, or both weights and activations can be compressed to 8 bits~\cite{yao2022zeroquant, dettmers2022gptint, xiao2022smoothquant}.
In \sys, we compress both the weights and KV cache to 4 bits and show how to combine the compression with offloading to make further improvements.

Within broader domains, memory optimizations and offloading have been studied for training~\cite{huang2020swapadvisor,ren2021zero,steiner2022olla} and linear algebra~\cite{jia1981complexity,demmel2013communication}.

%% file: sec.background.tex
\section{Background: LLM Inference}
\label{sec:background}
In this section, we describe the LLM inference workflow and its memory footprint.

\textbf{Generative Inference.}
A typical LLM generative inference task consists of two stages: i) the \textit{prefill} stage which takes a prompt sequence to generate the key-value cache (KV cache) for each transformer layer of the LLM; and ii) the \textit{decoding} stage which utilizes and updates the KV cache to generate tokens step-by-step, where the current token generation depends on previously generated tokens.

For a particular inference computation, denote the batch size by $b$, the input sequence length by $s$, the output sequence length by $n$, the hidden dimension of the transformer by $h_1$, the hidden dimension of the second MLP layer by $h_2$, and the total number of transformer layers by $l$. Given the weight matrices of a transformer layer specified by $\wk^i$, $\wq^i$, $\wv^i$, $\wo^i$, $\wa^i$, $\wb^i$, where $\wk^i, \wq^i, \wv^i, \wo^i \in \mathcal{R}^{h_1 \times h_1} $, $\wa \in \mathcal{R}^{h_1 \times h_2}$, and $\wb \in \mathcal{R}^{h_2 \times h_1}$.

During the \textit{\textbf{prefill phase}}, the input of the $i$-th layer is specified by $\x^i$, and \texttt{key}, \texttt{value}, \texttt{query}, and \texttt{output} of the attention layer is specified by $\xk^i$, $\xv^i$, $\xq^i$, $\xo^i$, where $\x^i, \xk^i, \xv^i, \xq^i, \xo^i \in \mathcal{R}^{b \times s \times h_1}$. Then, the cached \texttt{key}, \texttt{value} can be computed by:
\begin{equation*}
\begin{array}{cc}
     & \xk^i = \x^i \cdot \wk^i; \quad  \xv^i = \x^i \cdot \wv^i
\end{array}
\end{equation*}
The rest of the computation in the $i$-th layer is:
\begin{equation*}
\begin{array}{cc}
& \xq^i = \x^i \cdot \wq^i \\
& \xo^i = f_{\textsf{Softmax}}\left( \frac{\xq^i {\xk^i}^T}{\sqrt{h}}\right )\cdot \xv^i \cdot \wo^i + \x^i \\
& \x^{i+1} = f_{\textsf{relu}}\left(\xo^i \cdot \wa \right) \cdot \wb + \xo^i
\end{array}
\end{equation*}
During the \textit{\textbf{decode phase}}, given $\tx^i \in \mathcal{R}^{b \times 1 \times h_1}$ as the embedding of the current generated token in the $i$-th layer, the inference computation needs to i) update the KV cache:
\begin{equation*}
\begin{array}{cc}
     & \xk^i \leftarrow \textsf{Concat}\left( \xk^i, \tx^i \cdot \wk^i \right)\\
     & \xv^i \leftarrow \textsf{Concat}\left( \xv^i, \tx^i \cdot \wv^i \right)
\end{array}
\end{equation*}
 and ii) compute the output of the current layer:
\begin{equation*}
\begin{array}{cc}
     & \tq^i = \tx^i \cdot \wq^i \\
     & \txo^i = f_{\textsf{Softmax}}\left( \frac{\tq^i {\xk^i}^T}{\sqrt{h}}\right)\cdot \xv^i \cdot \wo^i + \tx^i\\
     & \tx^{i+1} = f_{\textsf{relu}}\left(\txo^i \cdot \wa \right) \cdot \wb + \txo^i
\end{array}
\end{equation*}

\textbf{Memory Analysis.}
The memory footprint of LLM inference mainly comes from the model weights and the KV cache. Considering the OPT-175B model in \texttt{FP16}, the total number of bytes to store the parameters can be roughly \footnote{We ignore the embedding layer(s), which is relatively small.}
 calculated by $l (8h_1^2+ 4h_1 h_2 )$.
The total number of bytes to store the KV cache in peak is $4 \times  b  l  h_1 (s+n) $.

In a realistic setting with a sufficient number of GPUs, the OPT-175B model ($l=96, h_1=12288, h_2=49152$) takes $325$ GB. With a batch size of $b=512$, an input sequence length $s=512$, and an output sequence length of $n=32$, the total memory required to store the KV cache is $1.2$ TB, which is $3.8\times$ the model weights, making the KV cache a new bottleneck of large-batch high-throughput inference.
In \sys, for OPT-175B, we enlarge the effective batch size to $256$ to achieve the throughput at $0.69$ token/s.

\textbf{Throughput and Latency.}
Considering an effective batch size $b$, an input sequence length $s$, and an output sequence length of $n$, the latency $t$ is defined as the total number of seconds spent to process the prompts and generate all the $bn$ tokens. The generation throughput is defined as $bn / t$.

%% file: sec.offload.tex
\begin{figure}[h]
\centering
\includegraphics[width=0.95\mywidth]{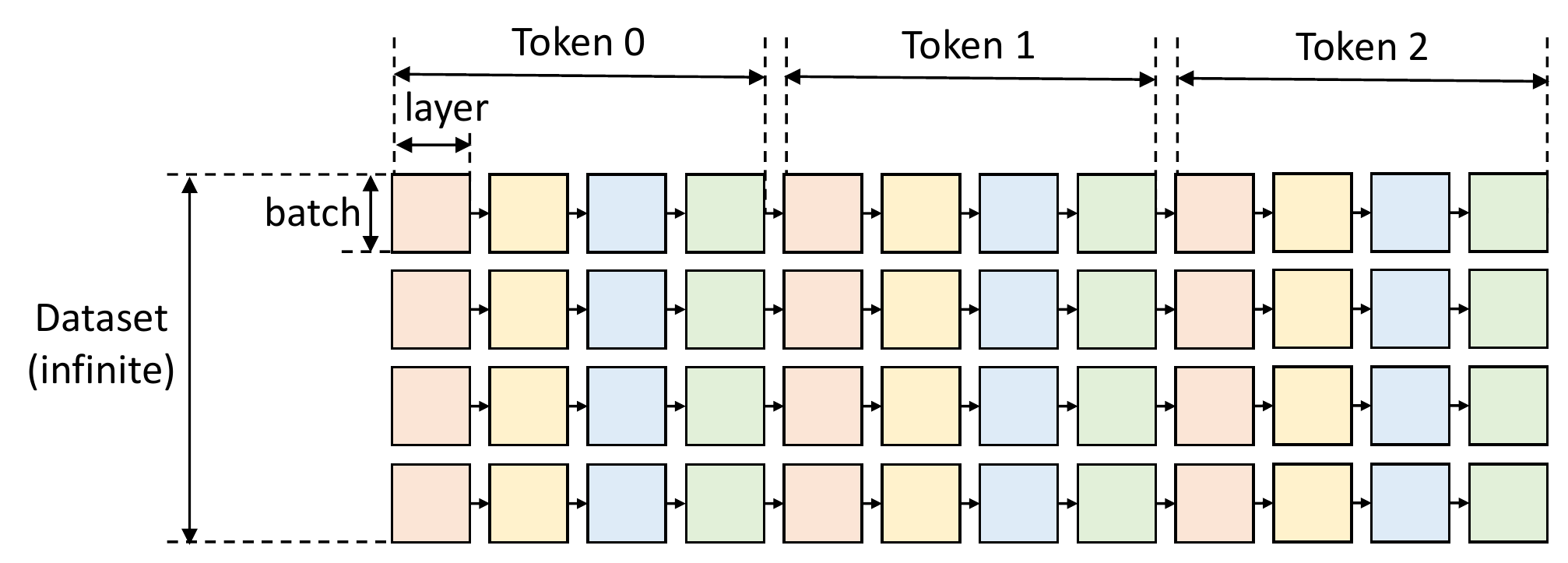}
\vspace{-1em}
\caption{Computational graph of LLM inference.} \label{fig:computational_graph}
\vspace{-0.5em}
\end{figure}

\section{Offloading Strategy}
\label{sec:offload}

In this section, we do not relax any computation of LLM inference and illustrate how to formalize the offloading procedure under the GPU, CPU, and disk memory hierarchy. We first formulate the problem and then construct the search space of the possible offloading strategies in \sys.
To find an efficient strategy, \sys builds an analytical cost model and searches for configurations with an optimizer based on linear programming.

\subsection{Problem Formulation}
\label{subsec:formulation}
Consider a machine with three devices: a GPU, a CPU, and a disk. The GPU and CPU can perform computation while the disk cannot. The three devices form a three-level memory hierarchy where the GPU has the smallest but fastest memory and the disk has the largest but slowest memory.
When an LLM cannot fit entirely within the GPU, we need to offload it to secondary storage and perform computation part-by-part by partially loading the LLM.

We formulate the generative inference with offloading as a graph traversal problem.
\cref{fig:computational_graph} shows an example computational graph, where the model has 4 layers and we generate 3 tokens per prompt. As our focus is throughput-oriented scenarios, we assume a given dataset with an infinite number of prompts that need to be processed.
In the figure, a square means the computation of a GPU batch for a layer. The squares with the same color share the same layer weights. We define a valid path as a path that traverses (i.e., computes) all squares, while subject to the following constraints:
\begin{itemize}
    \item A square can only be computed if all squares to its left on the same row were computed.
    \item To compute a square on a device, all its inputs (weights, activations, cache) must be loaded to the same device.
    \item After being computed, a square produces two outputs: activations and KV cache. The activations should be stored until its right sibling is computed. The KV cache should be stored until the rightmost square on the same row is computed.
    \item At any time, the total size of tensors stored on a device cannot exceed its memory capacity.
\end{itemize}

The goal is to find a valid path that minimizes the total execution time, which includes the compute cost and I/O cost when moving tensors between devices.

\subsection{Search Space}
\label{subsec:search-space}
Given the formulation above, we construct a search space for possible valid strategies in \sys.

\textbf{Compute schedule.}
Intuitively, there are two orders to traverse the graph in \cref{fig:computational_graph}: row-by-row and column-by-column.
All existing systems~\cite{aminabadi2022deepspeed, huggingfaceAccelerate} traverse the graph row-by-row, as shown in \cref{fig:block_schedule}(a). This is reasonable because it is the fastest way to finish the generation for one batch and the KV cache can be freed immediately after a row. However, because every two contiguous squares do not share weights, this schedule has to repeatedly load the weights and incurs huge I/O costs.

\ifisarxiv
\else

\begin{figure}[t]
\centering
\includegraphics[width=0.95\mywidth]{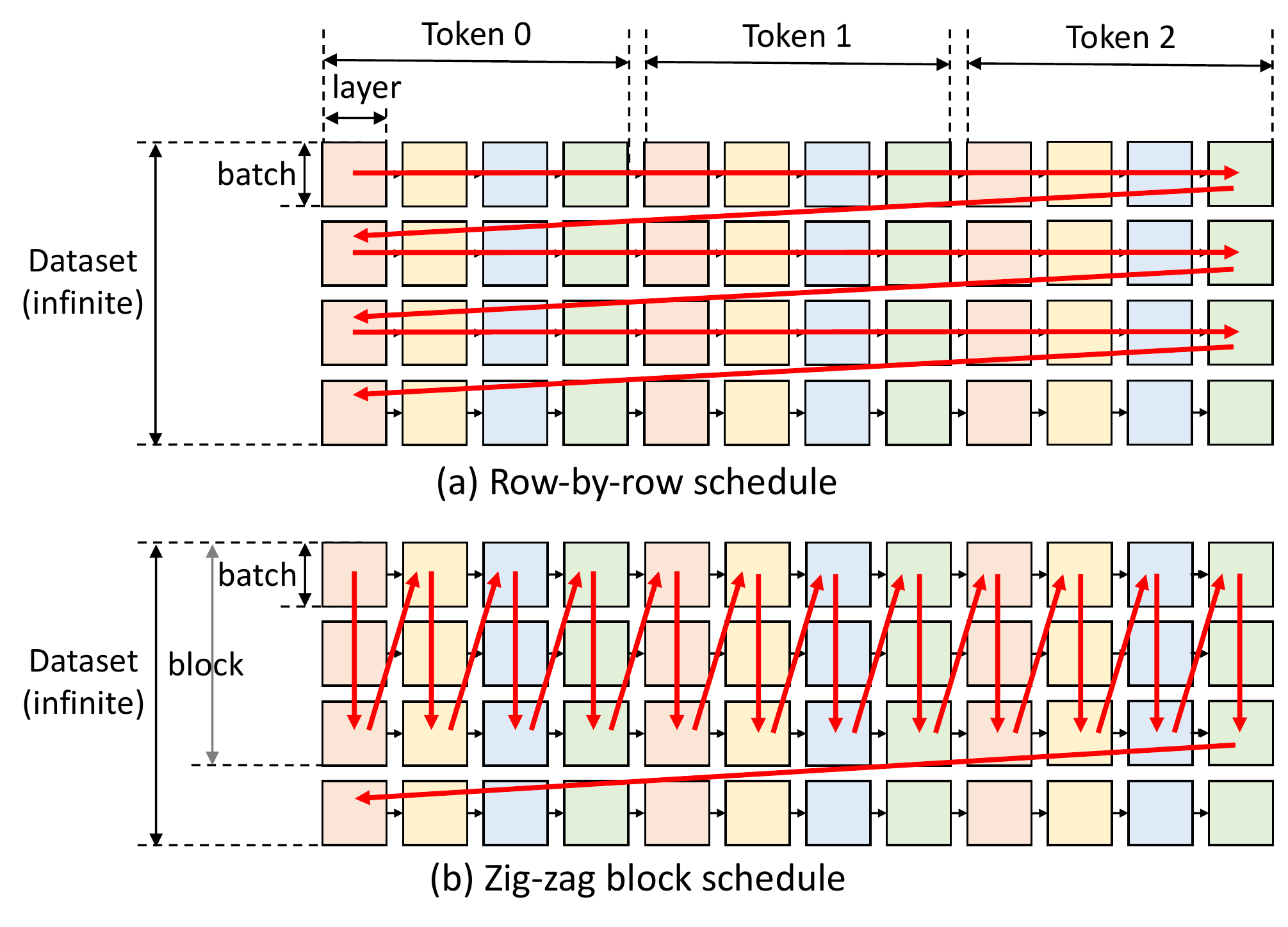}
\vspace{-0.5em}
\caption{Two different schedules. The red arrows denote the computation order.} \label{fig:block_schedule}
\end{figure}

\begin{figure}[t]
\vspace{-1em}
\begin{algorithm}[H]
\centering
\footnotesize
\caption{Block Schedule with Overlapping}\label{alg:block-schedule}
\begin{algorithmic}
\FOR{$i=1$ {\bfseries to} $generation\_length$}
  \FOR{$j=1$ {\bfseries to} $num\_layers$}
    \STATE // Compute a block with multiple GPU batches
    \FOR{$k=1$ {\bfseries to} $num\_GPU\_batches$}
      \STATE // Load the weight of the next layer
      \STATE \texttt{load\_weight}$(i, j+1, k)$
      \STATE // Store the cache and activation of the prev batch
      \STATE \texttt{store\_activation}$(i, j, k-1)$
      \STATE \texttt{store\_cache}$(i, j, k-1)$
      \STATE // Load the cache and activation of the next batch
      \STATE \texttt{load\_cache}$(i, j, k+1)$
      \STATE \texttt{load\_activation}$(i, j, k+1)$
      \STATE // Compute this batch
      \STATE \texttt{compute}$(i, j, k)$
      \STATE // Synchronize all devices
      \STATE \texttt{synchronize}$()$
    \ENDFOR
  \ENDFOR
\ENDFOR
\end{algorithmic}
\label{algo:overlap}
\end{algorithm}
\vspace{-2em}
\end{figure}
\fi

To reduce the I/O costs of the weights, we can traverse the graph column-by-column. All squares in a column share weights, so we can let the weights stay on GPU for reusing and only load/unload the activations and KV cache. However, we cannot traverse a column all the way to the end because the activations and KV cache still need to be stored. 
Hence, we have to stop when they fill the CPU and disk memory.
Taking all this into consideration, we converge to a zig-zag block schedule, as shown in \cref{fig:block_schedule}(b). Besides, we propose another more advanced and I/O-optimal schedule, but only implement the simpler block schedule due to the practical implementation difficulty of the optimal one. However, we prove that the block schedule is at most twice worse than the optimal schedule in \cref{subsec:optimality}.
\begin{restatable}{theorem}{suboptimal}
\label{thm:suboptimal}
The I/O complexity of the zig-zag \\ block schedule is within $2\times$ of the optimal solution.
\end{restatable}

Another typical optimization is overlapping. 
We can overlap the weights load of the next layer, cache/activation load of the next batch, cache/activation store of the previous batch, and the computation of the current batch.
Adding overlapping to the block schedule results in \cref{alg:block-schedule}. The first six functions in the innermost loop can be seen as launched in parallel with six logical threads because there are no dependencies. The last function then synchronizes these six logical threads. We rely on operating systems and CUDA drivers to resolve the schedule of the underlying hardware resources.
As a conclusion, the algorithm introduces two parameters into our search space: the GPU batch size and the number of GPU batches in a block. 
The product of the GPU batch size and the number of GPU batches is called block size (or \textbf{effective batch size}).

\textbf{Tensor placement.} Besides compute schedule, a strategy should specify how to store these tensors within the memory hierarchy.
We use three variables $wg$, $wc$, and $wd$ to define the percentages of weights stored on GPU, CPU, and disk respectively.
Similarly, we use three variables  $hg$, $hc$, $hd$ to define the percentages of activations and use $cg$, $cc$, $cd$ for the KV cache.
Given the percentages, there are still multiple ways to partition the tensors. Taking weight tensors as an example,
from coarse grain to fine grain, we can partition the weights at the model granularity (e.g., assign $50\%$ of the layers in a model to the GPU), at the layer granularity (e.g., assign $50\%$ of the tensors in a layer to the GPU), or at the tensor granularity (e.g., assign $50\%$ of the elements in a tensor to the GPU).
Coarser granularity leads to lower runtime overhead but it is less flexible and its cost is difficult to analyze. Considering both the runtime overhead and desired flexibility, we use layer granularity for weights, and tensor granularity for activations and the KV cache.

\textbf{Computation delegation.}
While CPUs are much slower than GPUs, we find using CPU compute can still be beneficial in some cases. This is because the computation of attention scores during decoding is I/O-bounded. Consider a case where the KV cache is stored on the CPU. Computing the attention scores on the GPU requires moving the entire KV cache to the GPU, which incurs a substantial I/O cost as the KV cache is huge.
In contrast, computing the attention score on the CPU does not require moving the KV cache. It only requires moving the activations from the GPU to the CPU. Quantitatively, let $b$ be the GPU batch size, $s$ be the sequence length, and $h_1$ be the hidden size. The size of the moved KV cache is $b \times s \times h_1 \times 4$ bytes, and the size of the moved activation is $b \times h _1\times 4$ bytes, so computing attention score on CPU reduces I/O by $s \times$. For long sequences (e.g., $s \geq 512$), it is better to compute the attention scores on the CPU if the associated KV cache is not stored on the GPU.

\subsection{Cost Model and Policy Search}
\label{subsec:cost-model}
The schedule and placement in \cref{subsec:search-space} constructs a search space with several parameters. Now we develop an analytical cost model to estimate the execution time given these algorithm parameters and hardware specifications.

\textbf{Cost Model.}
The cost model predicts the latency during prefill for one layer denoted as $T_{pre}$, and the averaged latency during decoding for one layer denoted as $T_{gen}$ in one block.
The total latency for computing a block can then be estimated as
$T=T_{pre} \cdot l + T_{gen} \cdot (n-1) \cdot l$, where $l$ is the number of layers and $n$ is the number of tokens to generate.

Assuming perfect overlapping, $T_{pre}$ can be estimated as
$T_{pre} = \max(ctog^p, gtoc^p, dtoc^p, ctod^p, comp^p)$,
where $ctog^p$, $gtoc^p$, $dtoc^p$, $ctod^p$, $comp^p$ denote the latency of read from CPU to GPU, write from GPU to CPU, read from disk to CPU, write from CPU to disk, computation, respectively, during prefill for one layer.

Similarly, $T_{gen}$ can be estimated as
$T_{gen} = \max(ctog^g, gtoc^g, dtoc^g, ctod^g, comp^g),$
with $ctog^g$, $gtoc^g$, $dtoc^g$, $ctod^g$, $comp^g$ denoting the latency of read from CPU to GPU, write from GPU to CPU, read from disk to CPU, write from CPU to disk, computation, respectively, during decoding for one layer.

For I/O terms like $dtoc^g$, it is estimated by summing up the I/O events, which contain weights, activations, and cache reads.
The size of \texttt{FP16} weights for one transformer layer is $8h_1^2 + 4h_1\cdot h_2$ bytes, with $h_1$ denoting the hidden size, and $h_2$ denoting the hidden size of the second MLP layer.
Let $bls$ be the block size and $s$ be the prompt length;
then the size of activations for one layer is $2\cdot bls \cdot h_1$.
The size of the KV cache for one layer on average is $4 \cdot bls\cdot (s+\frac{n}{2}) \cdot h_1$.
We have to load $wd, hd, cd$ percent of weights, activations, and the KV cache from the disk respectively so that the total latency of disk read is
$dtoc^g = \frac{1}{\text{disk\_to\_cpu\_bandwidth}}((8h_1^2 + 4h_1\cdot h_2)\cdot wd + 4 \cdot bls\cdot (s+\frac{n}{2}) \cdot h_1\cdot cd + 2\cdot bls \cdot h_1\cdot hd)$.

Similarly for computation terms, we sum up all computation events, including matrix multiplications and batched matrix multiplications on the CPU and the GPU.

Besides latency estimation, we also estimate the peak memory usage of the GPU, CPU, and disk, and then we add memory constraints. The full cost model is in \cref{sec:full-cost-model}.

\textbf{Policy Search.}
A policy includes 11 variables: block size $bls$, GPU batch size $gbs$, weight placement $wg, wc, wd$, activation placement $hg, hc, hd$, and KV cache placement $cg, cc, cd$.
In practice, the percentage cannot be an arbitrary real number between $0$ and $1$, because the tensor cannot be split arbitrarily.
However, we relax the percentage variables in the cost model to be any real number between $0$ and $1$ since it is changing gradually. We solve the problem as a two-level optimization problem. We first enumerate a few choices of $(bls, gbs)$ tuple. Typically, $gbs$ is a multiple of 4, and $bls$ is less than 20 so there are not too many choices.
Then with the fixed $bls, gbs$, finding the best placement $p=(wg, wc, wd, cg, cc, cd, hg, hc, hd)$ becomes a linear programming problem shown in \cref{eq:policy_lp}. The linear programming problem can be solved very quickly because there are only 9 variables.
This formulation can also be flexibly extended to include latency constraints and model approximate methods such as compression.

\vspace{-2em}
\begin{equation}
\begin{array}{rrclcl}
\displaystyle \min_{p} & \multicolumn{3}{c}{T/bls} \\
\textrm{s.t.} 
&gpu\ peak\ memory &<& gpu\ mem\ capacity\\
&cpu\ peak\ memory &<& cpu\ mem\ capacity\\
&disk\ peak\ memory &<& disk\ mem\ capacity\\
&wg+wc+wd&= &1\\
&cg+cc+cd&= &1\\
&hg+hc+hd&= &1
\end{array}
\label{eq:policy_lp}
\end{equation}
\vspace{-1em}

To use the cost model, we run profiling on the hardware to sample some data points and fit the hardware parameters. We then call the optimizer to get an offloading policy. Due to our relaxation and the hardness of accurately modeling peak memory usage (e.g., fragmentation), sometimes a strategy from the policy search can run out of memory. In this case, we manually adjust the policy slightly.
The cost model can usually return a good policy, but it is common that a better policy can be obtained by tuning manually.

\subsection{Extension to Multiple GPUs}
\label{subsec:multi-gpu}

We discuss how to extend the offloading strategy in \sys if there are multiple GPUs. Although we can find a nearly optimal strategy for one GPU, the strategy is still heavily limited by I/O and has a low GPU utilization.
If we are given more GPUs and more CPUs, model parallelism can be utilized to reduce the memory pressure of each GPU, which can potentially lead to a super-linear scaling in decoding.

There are two kinds of model parallelisms: tensor and pipeline parallelism~\cite{narayanan2021efficient,zheng2022alpa}. Tensor parallelism can reduce the single-query latency but pipeline parallelism can achieve good scaling on throughput due to its low communication costs. Since we target throughput, \sys implements pipeline parallelism.

We use pipeline parallelism by equally partitioning an $l$-layer LLM on $m$ GPUs, and then the execution of all GPUs follows the same pattern. The problem is reduced to running an $n/m$-layer transformer on one GPU. We can directly reuse the policy search developed for one GPU.
To achieve micro-batch pipelining, a new for-loop is added to \cref{alg:block-schedule} to combine the iteration-level pipeline parallel execution schedule \cite{huang2019gpipe, yu2022orca} with our single-device offloading runtime.

%% file: sec.approximate.tex
\section{Approximate Methods}
\label{sec:approximate}
The previous section focuses on the exact computation.
However, the inference throughput can be greatly boosted with negligible accuracy loss by allowing some approximations, because LLMs are typically robust to careful approximations.
This section introduces two such approximations: group-wise quantization and sparse attention.

\textbf{Group-wise Quantization.}
We show that both the weights and KV cache can be directly quantized into 4-bit integers without any retraining or calibration on OPT-175B, all while preserving similar accuracy (\cref{subsec:eval-approximation}).
When compared to some related works~\cite{yao2022zeroquant,dettmers2022gptint,xiao2022smoothquant} that try to use integer matrix multiplication mainly for accelerated computation,
the goal of quantization in our case is primarily for compression and reducing I/O costs.
Therefore, we can choose a fine-grained quantization format in favor of a high compression ratio and dequantize the tensors back to FP16 before computation.
We use a fine-grained group-wise asymmetric quantization method~\cite{shen2020q}.
Given a tensor, we choose $g$ contiguous elements along a certain dimension as a group. For each group, we compute the $min$ and $max$ of the group elements and quantize each element $x$ into $b$-bit integers by
$x_{quant} = round\left( \frac{x - min}{max - min} \times (2^b - 1)\right)$.

The tensors are stored in the quantized format and converted back to FP16 before computation.
Since both the weights and KV cache consume a significant amount of memory, we compress both to 4 bits with a group size of 64.
There are multiple ways to choose which dimension to group on.
We find that grouping the weights along the output channel dimension and the KV cache along the hidden dimension preserves the accuracy while being runtime-efficient in practice.
One thing to mention is that such a fine-grained group-wise quantization in \sys causes some overhead in compression and decompression.
Such an overhead could be very significant if run on a CPU which makes the CPU delegation useless, so we turn off the CPU delegation when enabling quantization.
A concurrent work~\cite{dettmers2022case} also finds that 4-bit precision is almost optimal for total model bits and zero-shot accuracy on OPT models. Compared to this previous work, we first propose to compress the KV cache and present the results on OPT-175B.

\textbf{Sparse Attention.}
We demonstrate that the sparsity of self-attention can be exploited by only loading the top 10\% attention value cache on OPT-175B, all while maintaining the model quality.
We present one simple Top-K sparse approximation. After computing the attention matrices, for each query, we calculate the indices of its Top-K tokens from the K cache. We then simply drop the other tokens and only load a subset of the V cache according to the indices.

The application of these approximations is straightforward. We present these preliminary but interesting results and intend to emphasize that \sys is a general framework that can seamlessly plug in many approximation methods.

%% file: sec.evaluation.tex
\section{Evaluation}
\label{sec:evaluation}

\begin{table}
\vspace{-1em}
\caption{Hardware Specs}
\label{table:hardware_specs}
\centering
\footnotesize
\begin{tabular}{lll}
\toprule
Device & Model & Memory \\
\midrule
GPU & NVIDIA T4 & 16 GB \\
CPU & Intel Xeon @ 2.00GHz & 208 GB\\
Disk & Cloud default SSD (NVMe) & 1.5 TB \\
\bottomrule
\end{tabular}
\vspace{-2em}
\end{table}

\textbf{Hardware.}
We run experiments on the NVIDIA T4 GPU instances from Google Cloud.
The hardware specifications are listed in \autoref{table:hardware_specs}.
The read bandwidth of SSD is about 2GB/s and the write bandwidth is about 1GB/s.
Our methods and implementations do not depend on specific hardware architectures. Some architecture (e.g. unified memory) could be more friendly to our method.
See \Cref{subsec:more-exp} for discussions and experiments on different hardware setups.

\textbf{Model.}
OPT models~\cite{zhang2022opt} with 6.7B to 175B parameters are used in the evaluation. Although we do not evaluate other models, the offloading  in \sys can be applied to other transformer LLMs, e.g., GPT-3~\cite{brown2020language}, PaLM~\cite{chowdhery2022palm}, and BLOOM~\cite{scao2022bloom} because they all share a similar structure.

\textbf{Workload.}
Our focus is high-throughput generation on a given dataset. We use synthetic datasets where all prompts are padded to the same length. The system is required to generate 32 tokens for each prompt. We test two prompt lengths: 512 and 1024 (for experiments in more settings, see \cref{subsec:more-exp}).
The evaluation metric is generation throughput, defined as the number of generated tokens / (prefill time + decoding time).
Sometimes running a full batch takes too long for certain systems — in this cases, we generate fewer tokens and project the final throughput.
We use dummy model weights in throughput benchmarks for all systems and real weights for accuracy evaluations.

\textbf{Baseline.}
We use DeepSpeed ZeRO-Inference~\cite{aminabadi2022deepspeed} and Hugging Face Accelerate~\cite{huggingfaceAccelerate} as baselines. They are the only systems that can run LLMs with offloading when there is not enough GPU memory.
DeepSpeed supports offloading the whole weights to the CPU or disk. It uses ZeRO data parallelism if there are multiple GPUs. Accelerate supports offloading a fraction of the weights.
It does not support distributed GPUs on different machines.
Both of them use the row-by-row schedule and can only put cache/activations on GPU. These systems support different quantization methods. However, the quantization in Accelerate is not compatible with offloading, and the quantization in DeepSpeed cannot preserve accuracy up to 175B, so we do not enable quantization on these systems.
In addition to offloading, decentralized collaborative inference is another option to lower the resource requirement for LLM inference. %
Thus, we also include Petals~\cite{borzunov2022petals, ryabinin2023swarm}
as an additional baseline.

\textbf{Implementation.}
\sys is implemented on top of PyTorch~\cite{paszke2019pytorch}. \sys manages multiple CUDA streams and CPU threads to overlap I/O with compute. \sys creates files for tensors stored on the disk and maps them as virtual memory to access them.

\subsection{Offloading}
\label{sec:eval_offloading}
\textbf{Maximum throughput benchmark.}
We first evaluate the maximum generation throughput the systems can achieve with one GPU on two prompt lengths. As shown in \cref{table:e2e_throughput_1_gpu}, \sys outperforms all baselines in all cases. On OPT-6.7B, Accelerate and \sys can successfully fit the whole model into a single GPU, so they choose to only use the GPU. DeepSpeed has a higher memory overhead and cannot fit OPT-6.7B into the GPU, so it uses slower CPU offloading. On OPT-30B, all systems switch to CPU offloading. DeepSpeed and Accelerate store the KV cache on the GPU, so they cannot use a very large batch size, while \sys offloads most weights and all KV cache to the CPU and enables a larger GPU batch size. In addition, \sys reuses the weights by block scheduling. On OPT-175B, all systems start to  offload the weights to the disk. Baseline systems can only use a maximum batch size of 2, but \sys can use a GPU batch size of 32 and a block size of $32\times 8$, achieving a $69\times$ higher throughput. 
With compression enabled, \sys achieves a $112\times$ higher generation throughput on a single GPU for prompt sequence length 512. This huge improvement is because \sys uses an effective batch size of 144 and compresses the weights and KV cache to fit into CPU memory to avoid slow disk swapping.
More details on the policy setups and effective batch sizes can be found in \cref{subsec:more-exp}.
More experiments on how disk specification affects the throughput see \Cref{subsec:more-exp}.

\cref{table:e2e_throughput_4_gpu} shows the results on 4 machines, with one GPU on each machine. OPT-30B or OPT-175B still cannot fit into 4 GPUs.
Naively, we can run 4 independent \sys in a data-parallel fashion to get a linear scaling on throughput. But here we show that pipeline parallelism can achieve super-linear scaling on decoding throughput.
With pipeline parallelism, the memory pressure of each machine is reduced so we can switch from small batch sizes to larger batch sizes, or switch from disk offloading to CPU-only offloading.
In \cref{table:e2e_throughput_4_gpu}, \sys does not achieve linear scaling on generation throughput (which counts both prefill and decoding time costs).
This is because there are pipeline bubbles during the prefill stage and our workload settings only generate 32 tokens.
However, \sys achieves super-linear scaling on decoding throughput (which only counts decoding time costs assuming the prefill is done). This means if we generate more tokens, pipeline parallelism will show its benefits as decoding time will dominate.

\begin{table}[t]
\centering
\vspace{-0.5em}
\caption{Generation throughput (token/s) of different systems. Accelerate, DeepSpeed, and \sys use 1 GPU. Petals uses 1 GPU for OPT-6.7B, 4 GPUs for OPT-30B, and 24 GPUs for OPT-175B, but reports per-GPU throughput. We benchmark Petals under a good network assumption with a delay of less than 10ms and bandwidth of 1 Gbps. The models are run in INT8 as the default for Petals. See \cref{subsec:collaborative} for more details about Petals. \sys is our system without compression; \sys (c) uses 4-bit compression. ``OOM'' means out-of-memory.}
\label{table:e2e_throughput_1_gpu}
\footnotesize
\resizebox{\columnwidth}{!}{
\begin{tabular}{lrrrrrr}
\toprule
Seq. length & \multicolumn{3}{c}{512} & \multicolumn{3}{c}{1024} \\
\cmidrule(lr){2-4}\cmidrule(lr){5-7}
Model size  & 6.7B   & 30B   & 175B   & 6.7B     & 30B     & 175B     \\
 \midrule
Accelerate  & 25.12  & 0.62  & 0.01   & 13.01    & 0.31    & 0.01     \\
DeepSpeed   & 9.28   & 0.60  & 0.01   & 4.59     & 0.29    & OOM      \\
Petals      & 8.25 & 2.84 & 0.08 & 6.56 & 1.51  & 0.06     \\
\sys        & 25.26  & 7.32  & 0.69   & 13.72    & 3.50    & 0.35     \\
\midrule
\sys (c)    & 29.12  & 8.70  & 1.12   & 13.18    & 3.98    & 0.42     \\
\bottomrule
\end{tabular}
}
\vspace{-1em}
\end{table}

\begin{table}[t]

\centering
\caption{The scaling performance on 4 GPUs. The prompt sequence length is 512. The number of GPUs is denoted in the parenthesis. Generation throughput (token/s) counts the time cost of both prefill and decoding while decoding throughput only counts the time cost of decoding assuming prefill is done.}
\label{table:e2e_throughput_4_gpu}
\footnotesize
\resizebox{\columnwidth}{!}{
\begin{tabular}{lrrrrrr}
\toprule
Metric & \multicolumn{3}{c}{Generation Throughput} & \multicolumn{3}{c}{Decoding Throughput} \\
\cmidrule(lr){2-4}\cmidrule(lr){5-7}
Model size       & 6.7B   & 30B   & 175B  & 6.7B   & 30B    & 175B   \\
\midrule
\sys(1)   & 25.26  & 7.32  & 0.69  & 38.28  & 11.52  & 0.83   \\
\sys(4)   & 201.12 & 23.61 & 2.33  & 764.65 & 48.94  & 3.86   \\
DeepSpeed (4)  & 50.00  & 6.40  & 0.05   & 50.20  & 6.40   & 0.05    \\
\bottomrule
\end{tabular}
}
\vspace{-1em}
\end{table}

\textbf{Latency-throughput trade-off.}
We configure these systems to achieve maximum throughput under various latency constraints and draw their latency-throughput trade-off curves in \cref{fig:throughput_vs_latency}. \sys sets a new Pareto-optimal frontier that significantly outperforms baselines. On the low-latency side, \sys supports partial offloading and uses more space for weights. On the high-throughput side, \sys aggressively offloads all things out of the GPU to achieve a large GPU batch size and block size.
Given the same latency requirement of 5000 seconds, \sys without compression can achieve a $40\times$ higher throughput compared to DeepSpeed and Accelerate.
If allowing a higher latency and compression, \sys can further boost throughput and reach a $100\times$ improvement by using an effective batch size of 144. In this case, compression enables \sys to fit all things in the CPU memory and avoid disk I/O. The detailed latency, throughput, and policy setup can be found in \cref{subsec:more-exp}.

\textbf{Runtime breakdown.}
We shows the runtime breakdown of OPT-175B on \sys in \cref{table:breakdown} in \cref{subsec:more-exp}.
We disable overlapping and profile the time used for major components.
The GPU compute utilization is 82\% and 13\% for prefill and decoding, respectively.

\textbf{Ablation study.}
We then isolate the improvement brought by each individual technique.
\cref{table:tech_ablation} lists the throughput \sys can achieve if disabling one technique at a time.
On OPT-30B, with all optimizations enabled, we put $20\%$ weights on GPU, $80\%$ weights on CPU, and all activations and KV cache to CPU. We also choose a GPU batch size of $48$ and a block size of $48\times3$. ``No policy search'' illustrates the performance of worse strategies, showing the importance of a good policy.
On both models, using CPU compute and overlapping brings non-trivial improvement.
We also port the policy used in DeepSpeed/Accelerate into \sys runtime, showing the suboptimality of their policy. A more detailed ablation study can be found in \cref{subsec:more-exp}.

\textbf{HELM and Data wrangling.} We tested the interaction of \sys and HELM~\cite{liang2022holistic} by evaluating a new model OPT-IML-30B~\cite{iyer2022opt}, which has not been included in the official release of HELM. \sys finishes the benchmark of 7 representative sub-scenarios in 21 hours
, with all system overhead included, under the hardware setup described in \cref{table:hardware_specs}.
\cref{table:helm_integration} in \cref{subsec:more-exp} shows the details of the tasks and the corresponding running time.
We also use \sys to run the data wrangling tasks~\cite{narayan2022can} with OPT models. The detailed task configurations and running time are in \cref{subsec:more-exp}.

\begin{table}[t]
\centering
\vspace{-0.5em}
\caption{Ablation study of proposed techniques. The numbers are generation throughput on 1 GPU with prompt length 512. The gray tuple denotes a policy (GPU batch size $\times$ \#GPU-batch, $wg$, $wc$). More see \Cref{subsec:more-exp}.}
\label{table:tech_ablation}
\footnotesize
\resizebox{\columnwidth}{!}{
\begin{tabular}{l|lll}
\toprule
Model size          & 30B     & 175B    \\
\midrule
All optimizations   & 7.32 \color{gray}{(48$\times$3, 20, 80)} & 0.69 \color{gray}{(32$\times$8, 0, 50)} \\
No policy search    & 7.26 \color{gray}{(48$\times$3, 0, 100)} &  0.27 \color{gray}{(32$\times$1, 0, 50)}  \\

No overlapping      & 5.86   & 0.59     \\
No CPU compute      & 4.03   & 0.62     \\
No disk             & 7.32   & OOM      \\
w/ DeepSpeed policy & 1.57   & 0.01     \\
\bottomrule
\end{tabular}
}
\vspace{-2em}
\end{table}

\begin{table}[t]
\centering
\vspace{-0.5em}
\caption{The accuracy (higher is better) and perplexity (lower is better) with approximate methods.}
\label{table:approximations_accuracy}
\footnotesize
\resizebox{\columnwidth}{!}{
\begin{tabular}{lllllll}
\toprule
Dataset  & \multicolumn{3}{c}{Lambada (acc)} & \multicolumn{3}{c}{WikiText (ppl)} \\
\cmidrule(lr){2-4}\cmidrule(lr){5-7}
Config & FP16  &  4-bit  & 4-bit-S &  FP16  & 4-bit  & 4-bit-S \\
\midrule
OPT-30B    & 0.725 &  0.724 & 0.718      &  12.72 & 12.90 &  12.90   \\
OPT-175B   & 0.758 &  0.756 & 0.756      &  10.82 & 10.94 &  10.94   \\
\bottomrule
\end{tabular}
}
\vspace{-2em}
\end{table}

\subsection{Approximations}
\label{subsec:eval-approximation}
We use two tasks to show that our approximation methods exhibit negligible accuracy loss: next-word prediction on Lambada~\cite{paperno2016lambada} and language modeling on WikiText~\cite{merity2016pointer}.
As shown in \cref{table:approximations_accuracy}, ``4-bit'' means using group-wise quantization to compress both weights and KV cache into 4-bit integers.
``4-bit-S'' means combining the quantization and sparse attention with a 10\% sparsity on the value cache. Both methods show negligible accuracy loss compared to \texttt{FP16}. 
The results reveal the robustness of LLMs against these approximations.
We also tried 3-bit compression but it cannot preserve accuracy.

\subsection{Offloading vs. Collaborative Inference}
\label{subsec:collaborative}
We compare \sys and Petals under different network conditions by setting a private Petals cluster on GCP with 4 nodes having one T4 GPU per node.
We use Linux traffic control to constrain the connections between instances to simulate a realistic decentralized network and benchmark the performance of an OPT-30B model (input sequence length: 512, output sequence length: 32).
We tune the batch size of each request to be 2 and issue requests by 6 parallel client processes to achieve the maximum throughput\footnote{The batch size of 1 did not result in a noticeably better latency.}.
In addition, we normalize the throughput of Petals by the number of used GPUs.
As shown in \cref{fig:vs_petals}, we find that the throughput of \sys with a single T4 outperforms the per-GPU throughput of the Petals cluster under all tested network conditions.
Petals does not utilize offloading, so it cannot use a very large batch size, which limits its scaling on throughput.
Thus, we believe offloading could be a more efficient solution for throughput than communicating a large volume of activations in a long decentralized pipeline; on the other hand, collaborative inference can be a more viable option in more latency-sensitive scenarios.

Interestingly, we find that \sys can achieve lower latency than Petals in slow networks with short generation. We speculate this is because the network bandwidth becomes the bottleneck for activation transfer, and a large delay incurs a significant overhead on each communication step in the pipeline.
For the curve of a 100ms delay network, we can observe a cross point between \sys and Petals.
This is because the activations during prefill are larger than the activations during decoding by a factor of the input sequence length.
Thus, the communication overhead is proportionally larger,
which significantly slows down Petals during prefill.

\begin{figure}[ht]
\centering
\resizebox{\columnwidth}{!}{
\vspace{-1em}
\includegraphics[width=1\mywidth]{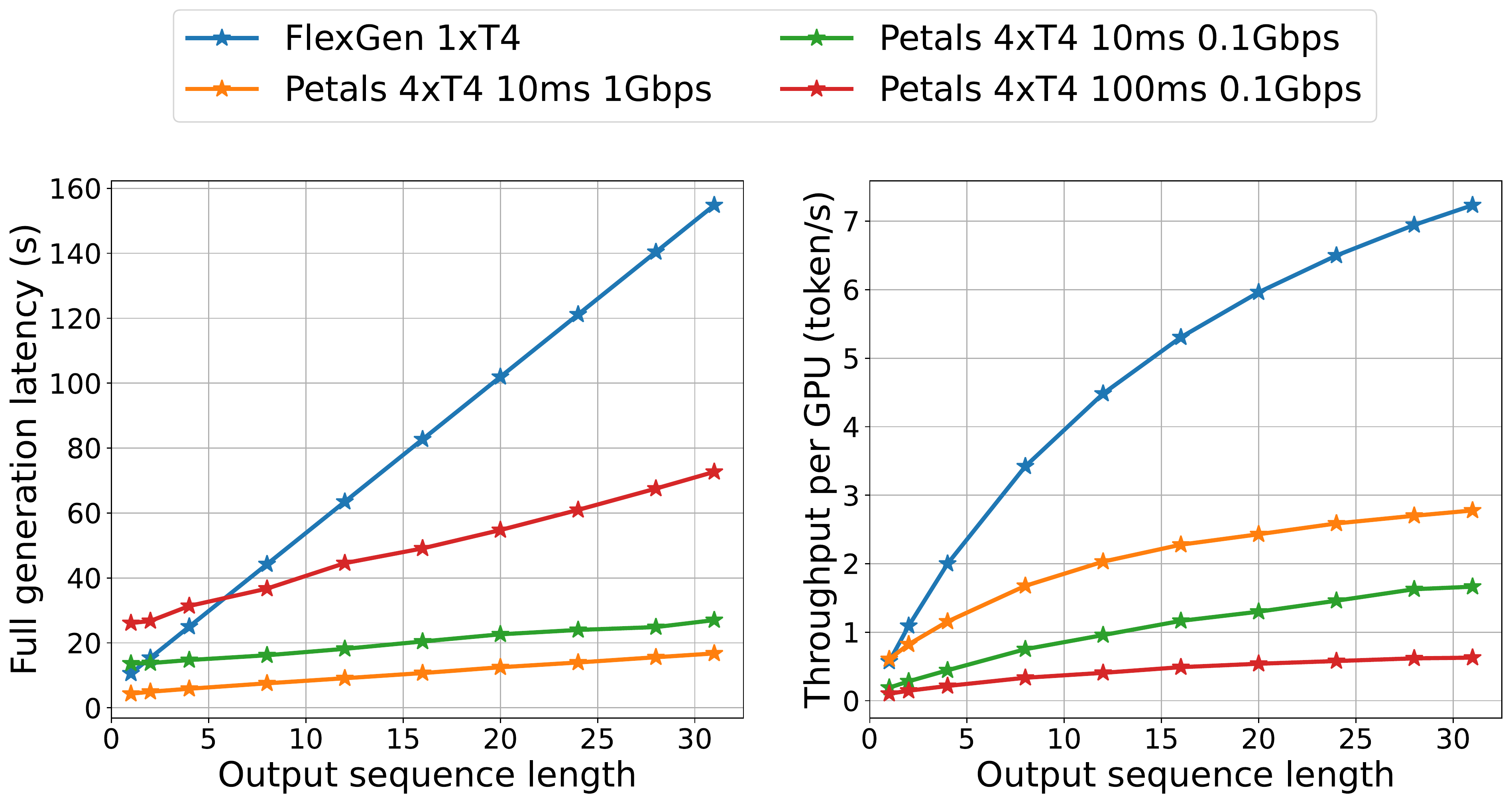}	
}
\vspace{-1.5em}
\caption{Full latency and per-GPU throughput of \sys and Petals in different network delay and bandwidth.}
\label{fig:vs_petals}
\vspace{-1em}
\end{figure}

%% file: sec.conclusion.tex
\section{Conclusion}
\label{sec:conclusion}
We introduce \sys, a high-throughput generation engine for LLM inference, which focuses on latency-insensitive batch-processing tasks for resource-constrained scenarios.

\vspace{-0.5em}
\section*{Acknowledgements}
We would like to thank Clark Barrett and Joseph E. Gonzalez for funding support, and Zhiqiang Xie, Daniel Y. Fu, Hao Zhang, Nick Chow, Benjamin Spector, Guangxuan Xiao, Jue Wang, Arjun Desai, Yao Fu, Anjiang Wei, and Zihao Ye for their insightful review and discussions.

%% file: sec.appendix.tex
\section{Appendix}
\label{sec:appendix}

\subsection{Notations}
\label{subsec:appendix-notation}
We use notations in \Cref{tab:model_vars} in this appendix.

\begin{table}[ht]
\centering
\begin{tabular}{|c|p{55mm}|}
  \hline
  Var  & Meaning \\ \hline
  $l$ & number of layers in the model\\ \hline
  $s$ & prompt sequence length\\ \hline
  $n$ & output sequence length\\ \hline
  $bls$ & block size\\ \hline
  $h_1$ & hidden size\\ \hline
  $h_2$ & hidden size of the second MLP layer\\ \hline
  $nh$ & number of head in the model\\ \hline
\end{tabular}
\caption{Notations}
\label{tab:model_vars}
\end{table}

\subsection{Compute Schedule Optimality}
\label{subsec:optimality}

This subsection discusses the graph traversal problem described in \Cref{subsec:formulation} and only considers the case that the model cannot fit in a single GPU.
We assume no application of CPU computation. To compute a square, the GPU loads the tensors it needs and offloads the cache and activations when finished.
We will analyze two schedules: the zig-zag block schedule used in \cref{subsec:search-space} and an I/O-optimal diagonal block schedule introduced in this section.
Note that our analysis only considers the theoretical I/O complexity.
In the real system, the latency and memory consumption cannot be the same as in the theoretical calculations.

There are three things that need to be stored during the generation process: weights, activations, and the KV cache.
From the computational graph, we have three observations.
(1) Suppose we need to swap the weights in and out of the GPU. Whatever the portion is, to finish the generation for one prompt, we need to swap $n$ times for $n$ tokens.
Therefore, it would be preferable to reuse the loaded weights for a batch of prompts, amortizing the weights I/O time.
(2) Each square will output activations which will be fed into the next layer.
Each row in the computational graph only needs to hold activations for one square at the same time.
(3) For each square besides the last $l$ squares in a row, the KV cache dumped by the square cannot be released until generating the last token (the last $l$ columns in the computational graph).
It is not shared across rows or columns, which will be the major factor in limiting the batch size.

\subsubsection{Zig-zag block schedule and Diagonal block schedule}

\textbf{Zig-zag block schedule.}
Inspired by the three observations introduced in 
\Cref{subsec:search-space},
we compute the first column in the computational graph for $bls$ samples,
save the dumped caches and activations,
then compute the second column for $bls$ samples, until the last column for $bls$ samples.
We call $bls$ as the block size as introduced in \Cref{subsec:search-space}. The computed $bls \cdot n \cdot l$ squares are called a block.

Assume \texttt{FP16} precision, to generate $n\cdot bls$ tokens during one block computation, we have to load $n$ times the whole model weights, do I/O operations on activations with $2(2h_1\cdot s \cdot bls\cdot l + 2h_1 \cdot bls \cdot l \cdot (n-1))$ bytes in total,
and do I/O on the KV cache with
$4h_1\cdot bls\cdot l \cdot (s \cdot n + n (n - 1) / 2)$ bytes in total.

Let $w$ denote the size of one-layer weights.
The peak memory used to store the weights, activations, and KV caches can be estimated as
\[ \text{peak\_mem} = w + 2h_1\cdot bls + 4h_1\cdot bls\cdot l\cdot (s+n) \]

If we only swap with CPU, then there is the constraint that peak\_mem $<$ CPU memory - some overhead.
Let $cmem$ denote the right hand, there is
\[ bls \leq \frac{cmem - w}{2h_1 + 4h_1\cdot l\cdot (s+n)} = bls_1 \]

Now we show that there is a better schedule that gives the same I/O efficiency but can enlarge the $bls$ by around 2 in some cases.

\paragraph{Diagonal block schedule}

\begin{figure}[t]
    \centering
    \includegraphics[width=\mywidth]{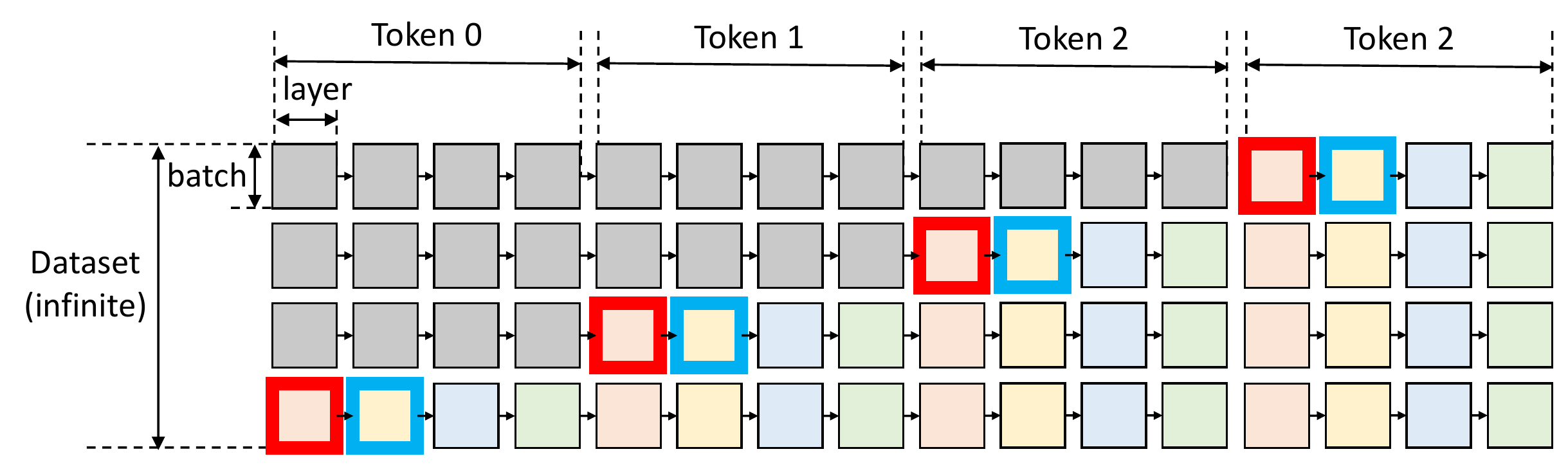}
    \caption{diagonal block schedule}
    \label{fig:diagonal-block}
\end{figure}

\Cref{fig:diagonal-block} is an illustration of our diagonal block schedule.
We have a block containing 4 GPU batches, and we are going to generate 4 tokens with a model that has 4 layers.
There will be a one-time warm-up phase (gray area) to compute the area above the diagonal.
Then for each iteration, the system will compute a diagonal that contains 4 sub-diagonals (4 squares enclosed by red outlines as the first sub-diagonal, then 4 squares enclosed by blue outlines as the second sub-diagonal).
After finishing the 4 sub-diagonals, it will repeat the same computation in the next row.

For simplicity, consider the good case that the memory capacity is large enough that the diagonal can cover all $n$ generation iterations for $n$ tokens.
The block size $bls$ now is defined as the number of samples touched by the diagonal.

In total, to compute one diagonal, the weights of each layer will be loaded once, and the I/O of the activations and KV cache will be in size roughly as $1/n$ as the value in the zig-zag block schedule.
There will be $bls$ tokens generated.
So the I/O per token is the same with the zig-zag block schedule after the one-time warm-up if for the same $bls$.

The peak memory needed to hold the necessary weights, activations, and KV cache is estimated as

\begin{align*}
    \text{peak\_mem} &= w + 2h_1\cdot bls\\
    &+ \frac{4h_1\cdot bls\cdot l (2s+n)(n-1)}{2n}
\end{align*}
from $\text{peak\_mem}\leq cmem$, we have
\[ bls\leq \frac{n(cmem - w)}{2h_1\cdot n +  2h_1\cdot l \cdot (2s+n)(n-1)} = bls_2 \]

Despite a one-time warm-up at the beginning.
The diagonal block schedule can accommodate a larger block size than zig-zag block schedule at the ratio of
\[ \frac{bls_2}{bls_1} = \frac{2s+2n}{2s+n} + O\left(\frac{1}{n}\right) \]
which is close to 2 when $n\gg s$, and close to 1 when $s\gg n$.

A larger $bls$ does not change the activations and KV caches I/O per token, but can reduce the weights I/O per token proportionally,
while weights I/O can normally occupy a large portion.

\textbf{Discussions.}
In offloading setting, I/O is a significant bottleneck in latency and throughput, so the diagonal block schedule should be able to give considerable gain when $n$ is relatively large compared to $s$ and the memory is sufficiently large to fit $n$ samples.

When the compute resources are sufficient to avoid offloading, the diagonal block schedule can still help to reduce the peak memory and enlarge the batch size, which increases GPU utilization.

Another benefit compared to the zig-zag block schedule is that with the same throughput, the generation latency for each prompt is reduced.
For example, suppose in the zig-zag block schedule the $bls$ samples finish the generation at the same time with latency $T$.
In the diagonal block schedule, the first $bls/n$ samples finish the generation with latency $T/n$, the second $bls/n$ samples finish with latency $2T/n$, and so on.
The average latency of completion is reduced by half.

Despite its advantages, there are some difficulties in implementing the diagonal block schedule. The major implementation difficulty is the dynamic update of the KV cache buffer. To improve runtime efficiency, \sys now pre-allocates continuous buffers for all KV cache at the beginning of a block. This works well for the zig-zag block schedule. However, for the diagonal block schedule, pre-allocating continuous buffers make it impossible to save memory anymore. To utilize the memory-saving property of the diagonal block schedule, one needs to implement efficient attention computation on non-contiguous memory.

\subsubsection{Proof of \Cref{thm:suboptimal}}

Note that in any case when we move from computing a square to another square, we need to offload and load the corresponding KV cache. So that the total I/O incurred by KV cache is constant.
The total I/O incurred by activations could vary, but despite the prefill phase, its size for each square is much smaller than the KV cache for the same square. In total, the size of activations is around $1/(2s+n)$ of the size of KV cache.
We will ignore the I/O incurred by activations for simplicity, which can cause a multiplicative error of $1/(2s+n)$ at most.
Then the only thing left is the weights I/O.
Starting from now, the I/O complexity in the context refers to the I/O complexity incurred by weights.

\begin{definition}
We define the working state at any time when the GPU is computing a square as follows.
Suppose there are $k$ GPU batches working in progress.
The column indices of the last squares that have been computed (including the current one) are $a_1, a_2, ..., a_k$, and $1 \le a_i \le n \times l$.
Different batches are identically independent, so w.l.o.g., suppose $a_1\geq a_2\geq ...\geq a_k$.
Then the working state is a tuple $(a_1, a_2,..., a_k)$.
A move that does a computation on a square is a pair of states $s^{(1)}, s^{(2)}$ that means transit from state $s^{(1)}$ to $s^{(2)}$.
\end{definition}

Consider an optimal order denoted as an infinite sequence $m_1, m_2, ...., m_{\infty}$, where $m_i$ is the $i$th move.
For each $i$, let $s_i$ be the current working state.

\begin{lemma}
\label{lem:uniform-delta}
If there is a list of moves that start from state $s$, and back to state $s$ at the end,
the number of computed squares for every column (one layer for one token) is the same.
\end{lemma}

\begin{proof}
Suppose the start state $s=(a_1, a_2,..., a_k)$.
For computations that occupy the whole row, the number of computed squares for every column is the same.
So we only need to consider the rows that have not been fully traversed (captured by the end state).
For each $a_i$, if the underlying row has not been finished at the end, and ends with the index $b_i$, then we pair $a_i$ with $b_i$.
If the underlying row has been finished, we pair it with a newly opened but not finished row, still, let $b_i$ denote the new index.

Thus we have transited from state $S_a = (a_1, a_2,..., a_k)$ to another state $S_b=(b_1, b_2,..., b_k)$. The indices in $S_a$ are sorted by $a_1\geq a_2\geq ...\geq a_k$. The indices in $S_b$ are not sorted, but $b_i$ is paired to $a_i$ according to the above paragraph.
For each $i$, if $b_i > a_i$, we need to count the squares in $(a_i, b_i]$ by 1.
If $b_i < a_i$, we need to count the squares in $(b_i, a_i]$ by -1.
Now we argue that for each column index $j$ and $1 \le j \le n \times l$, the count over it is summed to 0.
Suppose not, that there are $p$ positive count and $q$ negative count and $p\neq q$.
Then there are $p$ values lower than $j$ in state $a$ and $q$ values lower than $j$ in state $b$.
This contradicts the fact that $S_a$ and $S_b$ are the same state with different orders.
Therefore, the number of computed squares for every column is the same.
\end{proof}

\begin{theorem}
\label{thm:optimal}
The diagonal block schedule is I/O-optimal asymptotically.
\end{theorem}

\begin{proof}
Notice that since the memory capacity is finite, the length of the state is finite, thus the number of the possible state is finite.
If each state appears finite times in the sequence, then the sequence cannot be infinite.
Therefore, there exists a state $s$ that appears in the sequence infinite times.

Let $j_1, j_2,..., j_{\infty}$ be the indices in the sequence that have state $s$.
The moves between each two neighboring $s$ states correspond to a throughput.
The moves between $j_1$ and $j_2$ should create the highest possible throughput that pushes from state $s$ to $s$.
Otherwise, we can replace it to get a higher total throughput, which contradicts to that it is an optimal order.
So that we can repeat such a strategy between each neighboring $j_i, j_{i+1}$ to get an optimal compute order.

Now the problem is reduced to finding an optimal compute order between $j_1$ and $j_2$.
With infinite loops, the highest throughput from $j_1$ to $j_2$ gives the highest throughput among the whole sequence.

Assume an optimal compute order between $j_1$ and $j_2$. From \Cref{lem:uniform-delta}, there is the same number of squares to be computed for every column denoted as $c$.
With such fixed $c$, the throughput is determined by the I/O time between $j_1$ and $j_2$.
The number of times we load weights for each color in \Cref{fig:computational_graph} determines the total I/O time.
Each time we load weights, for example, the weights for computing the yellow squares, we cannot compute two yellow squares in the same row without other weights swaps, because the squares between them have not been computed and require other weights.

Therefore, for one load, we can only compute squares from different rows, which means all the caches and activations corresponding to those squares need to be held (either on the CPU or on the disk).
Every square corresponds to some memory consumption, for example, the squares in the range of the $i$-th token cost caches for $s+i-1$ tokens.
The sum of the memory consumption of all squares is a constant denoted as $M$.
Let $M'$ denote the memory capacity. The number of weights loading times is at least $\lceil M/M' \rceil$.
Let $t_w$ denote the I/O time for loading weights for one color, the optimal throughput is at most $c/\lceil M/M' \rceil/t_w$.

In the diagonal block schedule, after warm-up, each time with the loaded weights, the peak memory is the sum of the memory consumption of each computed square, which is the same each time we load weights.
We can set it to hit $M'$\footnote{The size value is discrete, we cannot exactly hit $M'$, but with large enough parameters, such a gap could be set to only affect the total value by less than $1\%$. For example, the layer could be at the tensor level to make squares extremely fine-grained.}.
Take $c$ number of diagonals as the repeated list of moves denoted as $\vec{q}$. Set the starting state to be $s$ mentioned before, $\vec{q}$ will restore the state to $s$ by construction.
The number of weights loading times during $\vec{q}$ is $\lceil M/M' \rceil$, which meets the lower bound, and achieves the throughput upper bound $c/\lceil M/M' \rceil/t_w$.
The warm-up phase can be ignored in the setting of an infinite sequence.
In summary, the diagonal block schedule is I/O optimal asymptotically.
\end{proof}

The zig-zag block schedule is not optimal, as the peak memory consumption is not the same each time loading the weights. When computing the layers for the last token, the peak memory is scaled with $s+n-1$, while for the first token, it is scaled with $s$.
In order to let the former fit in $M'$, the latter must be smaller than $M'$.
But the memory consumption change is linear when generating the tokens, thus the average memory consumption for each weights loading can be pushed to at least $M'/2$.
From this, the zig-zag block schedule can achieve the throughput at least $c/\lceil M/(M'/2)\rceil /t_w$ which is $1/2$ of the throughput upper bound.
In the infinite sequence setting, this means the zig-zag block schedule can achieve an I/O complexity that is at most 2$\times$ optimal.
Therefore, we have:

\suboptimal*

\subsection{Cost Model}
\label{sec:full-cost-model}

In this section, we present the full cost model.
Note that we use a single variable to represent constants like bandwidth and TFLOPS to simplify the formulation below.
In real systems, these constants vary according to the total load. 
We handle such dynamics by using piece-wise functions and adding regularization terms.
We carefully model the dynamics by depending only on other constants (e.g., hidden size), so the optimization problem remains a linear programming problem with respect to policy variables.

\Cref{tab:model_vars} and \Cref{tab:cost_vars} give the meaning of constants used in the cost model.

\begin{table}[ht]
\centering
\begin{tabular}{|c|p{55mm}|}
  \hline
  Var  & Meaning \\ \hline
  $ctog\_bdw$ & CPU to GPU bandwidth\\ \hline
  $gtoc\_bdw$ & GPU to CPU bandwidth\\ \hline
  $dtoc\_bdw$ & disk to CPU bandwidth\\ \hline
  $ctod\_bdw$ & CPU to disk bandwidth\\ \hline
  $mm\_flops$ & GPU flops per second for matrix multiplication\\ \hline
  $bmm\_flops$ & GPU flops per second for batched matrix multiplication\\ \hline
  $cpu\_flops$ & CPU flops per second\\ \hline
  $wg$ & percentage of weights on GPU\\ \hline
  $wc$ & percentage of weights on CPU\\ \hline
  $wd$ & percentage of weights on disk\\ \hline
  $cg$ & percentage of KV cache on GPU\\ \hline
  $cc$ & percentage of KV cache on CPU\\ \hline
  $cd$ & percentage of KV cache on disk\\ \hline
  $hg$ & percentage of activations on GPU\\ \hline
  $hc$ & percentage of activations on CPU\\ \hline
  $hd$ & percentage of activations on disk\\ \hline
  
\end{tabular}
\caption{Notation Variables}
\label{tab:cost_vars}
\end{table}

The object is to maximize throughput (token/s), which is equivalent to minimizing the reciprocal (s/token).
Free variables are colored blue.

\paragraph{Objective}
\[ \text{Minimize\space\space\space} T/\textcolor{blue}{bls} \]
Then the following constraints describe the calculation of total latency:
\[ T = Tpre\cdot l + Tgen\cdot (n-1)\cdot l \]
\[ Tpre = \max(ctog^p, gtoc^p, dtoc^p, ctod^p, comp^p) \]
\begin{align*}
    ctog^p = &\text{ } \frac{weights\_ctog^p + act\_ctog^p}{ctog\_bdw}\\
    = &\text{ } \frac{1}{ctog\_bdw} ((\textcolor{blue}{wc}+\textcolor{blue}{wd})(8h_1^2 + 4h_1\cdot h_2)\\
      &\text{\space\space\space\space\space\space\space\space\space\space\space\space\space\space\space\space\space\space} + 2(\textcolor{blue}{hc}+\textcolor{blue}{hd})s\cdot h_1\cdot \textcolor{blue}{bls})
\end{align*}
\begin{align*}
    gtoc^p = &\text{ } \frac{cache\_gtoc^p + act\_gtoc^p}{gtoc\_bdw}\\
    = &\text{ } \frac{1}{gtoc\_bdw} (4(\textcolor{blue}{cc}+\textcolor{blue}{cd})(s+1)h_1\cdot \textcolor{blue}{bls}\\  &\text{\space\space\space\space\space\space\space\space\space\space\space\space\space\space\space\space\space\space}+2(\textcolor{blue}{hc}+\textcolor{blue}{hd})s\cdot h_1\cdot \textcolor{blue}{bls})
\end{align*}
\begin{align*}
    dtoc^p = &\text{ } \frac{weights\_dtoc^p + act\_dtoc^p}{dtoc\_bdw}\\
    = &\text{ } \frac{1}{dtoc\_bdw} (\textcolor{blue}{wd}(8h_1^2 + 4h_1\cdot h_2)\\
      &\quad\quad\quad\quad\text{\space\space} + 2\textcolor{blue}{hd}\cdot s\cdot h_1\cdot \textcolor{blue}{bls})
\end{align*}
\begin{align*}
    ctod^p = &\text{ } \frac{cache\_ctod^p + act\_ctod^p}{ctod\_bdw}\\
    = &\text{ } \frac{1}{ctod\_bdw} (4\textcolor{blue}{cd}\cdot \textcolor{blue}{bls}\cdot (s+1)\cdot h_1\\
      &\quad\quad\quad\quad\text{\space\space} + 2\textcolor{blue}{hd}\cdot s\cdot h_1\cdot \textcolor{blue}{bls})
\end{align*}
\begin{align*}
    comp^p = &\text{ } \frac{linear\_layer^p}{mm\_flops} + \frac{att^p}{bmm\_flops}\\
    =&\text{ } \frac{\textcolor{blue}{bls}(8s\cdot h_1^2 + 4s\cdot h_1\cdot h_2)}{mm\_flops}\\
    &\text{ }+ \frac{4\textcolor{blue}{bls}\cdot s^2\cdot h_1}{bmm\_flops}
\end{align*}
\[ Tgen = \max(ctog^g, gtoc^g, dtoc^g, ctod^g, comp^g) \]
\begin{align*}
    ctog^g = &\text{ } \frac{weights\_ctog^g + act\_ctog^g}{ctog\_bdw}\\
    = &\text{ } \frac{1}{ctog\_bdw} ((\textcolor{blue}{wc}+\textcolor{blue}{wd})(8h_1^2 + 4h_1\cdot h_2)\\
      &\quad\quad\quad\quad\text{\space\space} + 2(\textcolor{blue}{hc}+\textcolor{blue}{hd})h_1\cdot \textcolor{blue}{bls})
\end{align*}
\begin{align*}
    gtoc^g = &\text{ } \frac{act\_gtoc^g}{gtoc\_bdw}\\
    = &\text{ } \frac{1}{gtoc\_bdw} (2(\textcolor{blue}{hc}+\textcolor{blue}{hd})\cdot h_1\cdot \textcolor{blue}{bls})
\end{align*}
\begin{align*}
    dtoc^g = &\text{ } \frac{cache\_dtoc^g +weights\_dtoc^g + act\_dtoc^g}{dtoc\_bdw}\\
    = &\text{ } \frac{1}{dtoc\_bdw} (4\textcolor{blue}{cd}\cdot \textcolor{blue}{bls}\cdot (s+n/2)\cdot h_1\\
      &\quad\quad\quad\quad\text{\space\space} +\textcolor{blue}{wd}(8h_1^2 + 4h_1\cdot h_2)\\
      &\quad\quad\quad\quad\text{\space\space} + 2\textcolor{blue}{hd}\cdot h_1\cdot \textcolor{blue}{bls})
\end{align*}
\begin{align*}
    ctod^g = &\text{ } \frac{cache\_ctod^g + act\_ctod^g}{ctod\_bdw}\\
    = &\text{ } \frac{1}{ctod\_bdw} (4\textcolor{blue}{cd}\cdot \textcolor{blue}{bls}\cdot h_1 + 2\textcolor{blue}{hd}\cdot h_1\cdot \textcolor{blue}{bls})
\end{align*}
\[ comp^g = gpu\_comp^g + cpu\_comp^g \]
\begin{align*}
    gpu\_comp^g = &\text{ } \frac{linear\_layer^g}{mm\_flops} + \frac{att^g}{bmm\_flops}\\
    =&\text{ } \frac{\textcolor{blue}{bls}(8h_1^2 + 4h_1\cdot h_2)}{mm\_flops}\\
     &\text{ }+ \frac{4\textcolor{blue}{cg}\cdot \textcolor{blue}{bls}\cdot (s+n/2)\cdot h_1}{bmm\_flops}
\end{align*}
\begin{align*}
cpu\_comp^g &= \frac{att^g}{cpu\_flops}\\
&= \frac{4(\textcolor{blue}{cc}+\textcolor{blue}{cd})\textcolor{blue}{bls}\cdot (s+n/2)\cdot h_1}{cpu\_flops}\\
\end{align*}

\paragraph{Peak Memory Constraints}

\begin{itemize}
    \item GPU peak memory constraints during prefill:
    
    GPU memory used to hold a fixed percentage of weights, activations, and cache is
    \begin{align*}
    gpu\_home^p = &\text{ } \textcolor{blue}{wg}\cdot (8h_1^2 + 4h_1\cdot h_2) \cdot l \\
    &\text{ }+ \textcolor{blue}{hg}\cdot 2s\cdot h_1 \cdot \textcolor{blue}{bls} \\
    &\text{ }+ 4(s+n)h_1\cdot \textcolor{blue}{cg}\cdot\textcolor{blue}{bls}\cdot l.
    \end{align*}
    
    GPU working memory (omit mask):
    \begin{align*}
    qkv^p =&\text{ }\textcolor{blue}{gbs} \cdot(2s\cdot h_1 + 3(2s\cdot h_1)) \\
    =&\text{ } \textcolor{blue}{gbs} \cdot 8s\cdot h_1\\
    att_1^p =&\text{ } \textcolor{blue}{cg}\cdot\textcolor{blue}{gbs} \cdot (2s\cdot h_1 + 2s\cdot h_1 + 2nh\cdot s^2) \\
    att_2^p =&\text{ }\textcolor{blue}{cg}\cdot\textcolor{blue}{gbs}\cdot( 2nh\cdot s^2+ 2s\cdot h_1 + 2s\cdot h_1) \\
    embed^p =&\text{ }\textcolor{blue}{gbs} \cdot(2s\cdot h_1 + 2s\cdot h_1)\\
    =&\text{ } \textcolor{blue}{gbs} \cdot 4s\cdot h_1 \\
    mlp_1^p =&\text{ } \textcolor{blue}{gbs} \cdot 2(s\cdot h_1 + s\cdot h_2) \\
    =&\text{ } 2\textcolor{blue}{gbs} \cdot s (h_1+h_2) \\
    mlp_2^p =&\text{ } \textcolor{blue}{gbs} \cdot 2(s\cdot h_2 + s\cdot h_1) \\
    =&\text{ } 2\textcolor{blue}{gbs} \cdot s (h_1 + h_2)
    \end{align*}
    \begin{align*}    
    gpu\_w^p&=\text{ } 2(1- \textcolor{blue}{wg})(8h_1^2+4h_1\cdot h_2) \\
    + &(1-\textcolor{blue}{hg})\cdot 2s\cdot h_1\cdot \textcolor{blue}{gbs} \\
       + &\max(
        qkv,
        att_1,
        att_2,
        embed,
        mlp_1,
        mlp_2)
    \end{align*}
    \[ gpu\_peak^p = gpu\_home^p + gpu\_w^p < gmem \]
    
    \item GPU peak memory constraints after prefill:
        
    GPU memory used to hold a fixed percentage of weights, activations, and cache is
    \begin{align*}
    gpu\_home^g = &\text{ } \textcolor{blue}{wg}\cdot (8h_1^2 + 4h_1\cdot h_2) \cdot l \\
    &\text{ }+ \textcolor{blue}{hg}\cdot 2 h_1 \cdot \textcolor{blue}{bls} \\
    &\text{ }+ 4(s+n)h_1\cdot \textcolor{blue}{cg}\cdot\textcolor{blue}{bls} \cdot l.
    \end{align*}
    
    GPU working memory (omit mask):
    \begin{align*}
    qkv^g =&\text{ }\textcolor{blue}{gbs} \cdot(2h_1 + 3(2h_1)) = 8\textcolor{blue}{gbs} \cdot h_1\\
    att_1^g =&\text{ } \textcolor{blue}{cg}\cdot\textcolor{blue}{gbs} \cdot (2h_1 +2(s+n) h_1 \\
    &\text{\quad\quad\quad\quad} + 2nh(s+n)) \\
    att_2^g =&\text{ }\textcolor{blue}{cg}\cdot\textcolor{blue}{gbs}\cdot( 2nh (s+n)+ 2(s+n)h_1 \\
    &\text{\quad\quad\quad\quad}+ 2h_1) \\
    embed^g =&\text{ }\textcolor{blue}{gbs} \cdot(2h_1 + 2h_1) = 4\textcolor{blue}{gbs} \cdot h_1 \\
    mlp_1^g =&\text{ } 2\textcolor{blue}{gbs} \cdot (h_1 + h_2) \\
    mlp_2^g =&\text{ } 2\textcolor{blue}{gbs} \cdot(h_2 + h_1)
    \end{align*}
    \begin{align*} 
    gpu\_w^g& =\text{ } 2(1- \textcolor{blue}{wg})(8h_1^2+4h_1\cdot h_2) \\
    +&(1-\textcolor{blue}{hg})\cdot 2s\cdot h_1\cdot \textcolor{blue}{gbs} \\
    +&\max(
    qkv^g,
    att_1^g,
    att_2^g,
    embed^g,
    mlp_1^g,
    mlp_2^g)
    \end{align*}
    \[ gpu\_peak^g = gpu\_home^g + gpu\_w^g < gmem \]
    
    \item CPU peak memory constraints during prefill:

    CPU memory used to hold a fixed percentage of weights, activations, and cache is
    \begin{align*}
    cpu\_home^p = &\text{ } \textcolor{blue}{wc}\cdot (8h_1^2 + 4h_1\cdot h_2) \cdot l \\
    + &\text{ } \textcolor{blue}{hc}\cdot 2s\cdot h_1 \cdot \textcolor{blue}{bls} \\
    + &\text{ } 4(s+n)h_1\cdot \textcolor{blue}{cc}\cdot \textcolor{blue}{bls} \cdot l.
    \end{align*}

    CPU working memory:
    \begin{align*}
    cpu\_w^p =&\text{ } (1-\textcolor{blue}{wg})(8h_1^2 + 4h_1\cdot h_2) \\
        &\text{ }+ (1-\textcolor{blue}{hg})\cdot 2s\cdot h_1\cdot \textcolor{blue}{gbs}.
    \end{align*}
    \[ cpu\_peak^p = cpu\_home^p + cpu\_w^p < cmem \]
    
    \item CPU peak memory constraints after prefill:
    
    CPU memory used to hold fixed percentage of weights, activations, and cache is
    \begin{align*}
    cpu\_home^g = &\text{ } \textcolor{blue}{wc}\cdot (8h_1^2 + 4h_1\cdot h_2) \cdot l \\
    &+ \textcolor{blue}{hc}\cdot 2h_1 \cdot \textcolor{blue}{bls} \\
    &+ 4(s+n)h_1\cdot \textcolor{blue}{cc}\cdot \textcolor{blue}{bls}\cdot l.
    \end{align*}

    CPU working memory:
    \begin{align*}
    cpu\_w^g =&\text{ } \textcolor{blue}{wd}(8h_1^2 + 4h_1\cdot h_2) \\
        &\text{ }+ 2\textcolor{blue}{hd}\cdot 2\cdot h_1\cdot \textcolor{blue}{gbs}\\
        &\text{ }+ 2\textcolor{blue}{cd}\cdot 4(s+n)h_1\cdot \textcolor{blue}{gbs}\\
        &\text{ }+ 2nh\cdot(s+n)\cdot \textcolor{blue}{gbs}\\
        &\text{ }+ 2h_1\cdot\textcolor{blue}{gbs}.
    \end{align*}
    \[ cpu\_peak^g = cpu\_home^g + cpu\_w^g < cmem \]
    
    \item NVMe peak memory constraints:
    \begin{align*}
    nvme\_peak = &\text{ } (8h_1^2 + 4h_1\cdot h_2) \cdot \textcolor{blue}{wd} \cdot l \\
    &\text{ }+ \textcolor{blue}{hd}\cdot 2s\cdot h_1\cdot \textcolor{blue}{bls} \\
    &\text{ }+ \textcolor{blue}{cd}\cdot 4(s+n)h_1\cdot \textcolor{blue}{bls} \cdot l \\
    <&\text{ } nmem
    \end{align*}
\end{itemize}

\subsection{Tables and Additional Experimental Results}
\label{subsec:more-exp}

\textbf{Execution Breakdown}
\cref{table:breakdown} shows the execution time breakdown for OPT-175B running on \sys with the setup in \cref{table:hardware_specs}.

\textbf{HELM and Data Wrangling}
\cref{table:helm_integration} lists the details of HELM integration experiments.
\cref{table:data_wrangling_30b} and \cref{table:data_wrangling_175b} shows additional results for the data wrangling task.

\textbf{Complementary Tables for Policy Details}
\cref{table:e2e_throughput_1_gpu_setup_512} and \cref{table:e2e_throughput_1_gpu_setup_1024} list the concrete policy setups for the results in \cref{table:e2e_throughput_1_gpu} for prompt length 512 and 1024, from end-to-end throughput experiments. \cref{table:latency_throughput_175} and \cref{table:latency_throughput_30} list the latency and throughput for the data points in \cref{fig:throughput_vs_latency} which demonstrate latency-throughput tradeoff.

\textbf{Abalation Study}
\Cref{table:main_ablation_policy} list the concrete policy setups for the main ablation study result in \Cref{table:tech_ablation}.
\cref{table:tech_ablation_throughput} and \cref{table:tech_ablation_latency} shows some additional ablation study on policies.
In \Cref{table:main_ablation_policy}, DeepSpeed chooses to store the KV cache and activations on GPU. For OPT-30B, the weights will be stored on the CPU entirely because it cannot fit in GPU. The corresponding percentage is $(0, 100, 100, 0, 100, 0)$. The computation order of DeepSpeed is row-by-row, so the number of GPU batches in a block is 1. The GPU batch size is set to be as large as possible, which is set to 8.
For OPT-175B, the weights will be stored on disk entirely according to DeepSpeed’s strategy, since it cannot be stored on CPU. The corresponding percentage is $(0, 0, 100, 0, 100, 0)$. The number of GPU batches in a block is 1, and the GPU batch size is 2.
For ``No policy search'', we use different policy changes for OPT-30B and OPT-175B to demonstrate the impact of different policy dimensions.
For OPT-30B, we change the percentage for weights from (20, 80) to (0, 100), and show that the throughput does not change much. For OPT-175B, we change the number of GPU batches in a block from 8 to 1 and show that the throughput degrades significantly.
For ``No CPU compute'', it degrades OPT-30B more than OPT-175B because the bottleneck for OPT-175B is on disk offloading. Therefore, the gain for CPU computation is small for OPT-175B. While for OPT-30B, the disk has not been used, so the gain for CPU computation is more significant.

\textbf{Different SSD Speed}
To highlight the limitation and requirements of SSD speed. We tested two kinds of disk on GCP and report the generation throughput (token/s) in \Cref{table:ssd_throughput} (input sequence length = 512 and output sequence length = 32).

\textbf{Additional Hardware and Sequence Length}
Our methods and implementations do not depend on specific hardware architectures. It can work well on different CPU architectures (e.g., Intel, AMD) and different GPU architectures (e.g., NVIDIA Ampere, NVIDIA Turing) as long as the architectures are supported by PyTorch. Some architecture (e.g. unified memory) could be more friendly to our approach.
To tune the system for different architectures, we need to fit a cost model and run policy search to generate offloading policies, which can be different according to the compute capabilities, memory capacities, and memory bandwidth of different architectures. The final absolute performance will vary, but FlexGen can be easily adapted to different architectures.
We did additional experiments on a different hardware setup of 24GB RTX 3090 with 125GB CPU Memory and 1TB SSD, in addition to our previous setting of 16GB T4 with 208GB CPU Memory and 1.5TB SSD, shown in \Cref{table:3090}. The input sequence length is set to 512 and the output sequence length is set to 32.
We can see the results follow similar trends to the setup in the main paper. FlexGen outperforms other baselines significantly. Comparing this 3090 setting with the T4 setting in the main paper, the performance under the 3090 setting is worse than the T4 setting for 30B and 175B. This is because CPU memory also plays a critical role when offloading is needed, making our T4 setting with larger CPU memory better.

\cref{table:e2e_throughput_1_gpu_setup_256} and \cref{table:more_seq_len_1_gpu} show the results for an additional prompt length 256.
As all of our benchmarks in the main paper are done with output sequence length 32, so we add two additional fixed sequence lengths in \Cref{table:e2e_throughput_1_gpu_setup_128_128} and \Cref{table:e2e_throughput_1_gpu_setup_512_8}.
The numbers are generally higher in the former one because the input sequence length is smaller and the output sequence length is larger. As the throughput is defined as (number of generated tokens) / (prefill time + generation time), such a setting makes the fraction of prefill time smaller.
The numbers are generally lower in the latter one because the output sequence length is smaller.

In summary, FlexGen outperforms baselines in all newly added settings. The Compression techniques used in FlexGen are helpful only for large models that need offloading. CPU memory capacity is essential for large models that need offloading.

\textbf{Batches with Various Sequence Length}
We also add experiments of one realistic use case with a mixture of prompt and output lengths (HELM benchmark) in \Cref{table:helm_mixed_length}.
To batch sequences of variable lengths, \sys simply pads all inputs to the maximum prompt length, which is a common method used in many systems. Depending on the distribution of the prompt length, the efficiency of this simple padding method varies. For example, if most sequences have similar lengths, then the baching efficiency should be very high. if some sequences are very long and some sequences are short, then \sys will spend a lot of time on the useless computation of padding tokens.
We use two metrics: padded throughput = (number of tokens in padded prompts + number of tokens in padded outputs) / latency and actual throughput = (number of non-padding tokens in prompts + number of non-padding tokens in outputs) / latency.
To better handle prompts with various lengths, one can utilize some complementary techniques from Orca\cite{yu2022orca}.

\input{tables/breakdown}

\input{tables/helm_run_time}
\input{tables/data_wrangling}

\input{tables/e2e_throughput_1_gpu_3090}

\input{tables/e2e_throughput_1_gpu_main_full}
\input{tables/e2e_throughput_1_gpu_policys}

\input{tables/e2e_throughput_1_gpu_setup_128_128}
\input{tables/e2e_throughput_1_gpu_setup_512_8}

\input{tables/pareto_full}

\input{tables/ablation_full}

\input{tables/main_ablation_policy}

\input{tables/ssd_throughput}

\input{tables/helm_mixed_length}

%% file: tables/breakdown.tex
\begin{table}[t]
\centering
\vspace{-0.5em}
\caption{Execution time breakdown (seconds) for OPT-175B. The prompt length is 512. (R) denotes read and (W) denotes write.}
\label{table:breakdown}
\footnotesize
\resizebox{\columnwidth}{!}{
\begin{tabular}{lllllll}
\toprule
Stage     & Total & Compute & Weight (R) & Cache (R) & Cache (W) \\
\midrule
Prefill   & 2711  & 2220    & 768        & 0         & 261       \\
Decoding  & 11315 & 1498    & 3047       & 7046      & 124       \\
\bottomrule
\end{tabular}
}
\vspace{-1.5em}
\end{table}

%% file: tables/helm_run_time.tex
\begin{table*}[t]
\centering
\caption{The setup and running time of 7 representative sub-scenarios in the HELM integration. The running time consists of dataset downloading, model initialization, generation, and metric computation.
``Prompt len'' denotes the input sequence length, and ``Gen len'' denotes the output sequence length. ``Num seq'' denotes the number of sequences (prompts). ``time'' denotes the running time in minutes.
}
\label{table:helm_integration}
\begin{tabular}{lrrrr}
\toprule
Scenario description & Prompt len & Gen len & Num seq &  time  \\
\midrule
wikifact: k=5, subject=plaintiff & 256 & 8 & 288 & 10 \\
wikifact: k=5, subject=instance\_of & 256 & 8 & 2592 & 55 \\
mmlu: subject=abstract\_algebra & 512 & 1 & 864 & 31 \\
mmlu: subject=us\_foreign\_policy & 512 & 1 & 1008 & 33 \\
synthetic\_reasoning: mode=pattern\_match & 256 & 50 & 1584 & 118 \\
synthetic\_reasoning\_natural: difficulty=easy & 512 & 20 & 1584 & 100 \\
summarization\_xsum: temperature=0.3 & 1984 & 64 & 1568 & 902 \\
\bottomrule
\end{tabular}
\end{table*}


%% file: tables/data_wrangling.tex
\begin{table*}[t]
\centering
\vspace{-0.5em}
\caption{The setup and running time of 6 representative data wrangling tasks with OPT-30B. Because the output seq. length is short for this task, we use a new metric total throughput = (number of tokens in the prompt + number of generated tokens) / total latency.}
\label{table:data_wrangling_30b}
\footnotesize
\begin{tabular}{rrrrrr}
\toprule
Task & Number of seq. & Input seq. length & Output seq. length & Running time (s) & Total throughput (token/s) \\
\midrule
EM: Fodors-Zagats      & 189            & 744           & 3             &541.550  & 248.287               \\
EM: Beer               & 91             & 592           & 3             &238.58   & 224.450               \\
EM: iTunes-Amazon      & 109            & 529           & 3             &267.639  & 198.775               \\
DI: Restaurant         & 86             & 123           & 5             &60.310   & 169.790               \\
DI: Buy                & 65             & 488           & 10            &185.882  & 160.747               \\
ED: Hospital           & 200            & 200           & 3             &158.329  & 256.429               \\
\bottomrule
\end{tabular}
\vspace{-0.5em}
\end{table*}

\begin{table*}[t]
\centering
\vspace{-0.5em}
\caption{The setup and running time of 6 representative data wrangling tasks with OPT-175B. Because the output seq. length is short for this task, we use a new metric total throughput = (number of tokens in the prompt + number of generated tokens) / total latency.}
\label{table:data_wrangling_175b}
\footnotesize
\begin{tabular}{rrrrrr}
\toprule
Task & Number of seq. & Input seq. length & Output seq. length & Running time (s) & Total throughput (token/s) \\
\midrule
EM: Fodors-Zagats      & 189            & 744           & 3             &3928.310  & 34.228               \\
EM: Beer               & 91             & 592           & 3             &1356.786  & 35.083               \\
EM: iTunes-Amazon      & 109            & 529           & 3             &1569.062  & 33.906               \\
DI: Restaurant         & 86             & 123           & 5             &648.762   & 16.968               \\
DI: Buy                & 65             & 488           & 10            &2086.961  & 14.317               \\
ED: Hospital           & 200            & 200           & 3             &1154.133  & 35.178               \\
\bottomrule
\end{tabular}
\vspace{-0.5em}
\end{table*}

%% file: tables/e2e_throughput_1_gpu_3090.tex
\begin{table*}[t]
\centering
\caption{Generation throughput (token/s) on 1 GPU (RTX 3090) with 125 GB CPU memory and 1TB SSD, run with \textbf{input sequence length 512 and output sequence length 32}. \sys is our system without compression; \sys (c) uses 4-bit compression.
The gray tuple denotes a policy (GPU batch size $\times$ \#GPU-batch, $wg$, $wc$, $cg$, $cc$, $hg$, $hc$).}
\label{table:3090}
\footnotesize
\begin{tabular}{llll}
\toprule
Seq. length & \multicolumn{3}{c}{512 + 32} \\
\cmidrule(lr){2-4}
Model size  & 6.7B   & 30B   & 175B  \\
 \midrule
Accelerate  & 183.177 \color{gray}{(16$\times$1, 100, 0, 100, 0, 100, 0)} & 2.077 \color{gray}{(13$\times$1, 0, 100, 100, 0, 100, 0)} & 0.026 \color{gray}{(4$\times$1, 0, 0, 100, 0, 100, 0)} \\
DeepSpeed   & 38.027 \color{gray}{(32$\times$1, 0, 100, 100, 0, 100, 0)} & 3.889 \color{gray}{(12$\times$1, 0, 100, 100, 0, 100, 0)} & 0.019 \color{gray}{(3$\times$1, 0, 0, 100, 0, 100, 0)} \\
\sys        & 233.756 \color{gray}{(28$\times$1, 100, 0, 100, 0, 100, 0)} & 5.726 \color{gray}{(4$\times$15, 25, 75, 40, 60, 100, 0)} & 0.384 \color{gray}{(64$\times$4, 0, 25, 0, 0, 100, 0)}\\
\midrule
\sys (c)    & 120.178 \color{gray}{(144$\times$1, 100, 0, 100, 0, 100, 0)} & 16.547 \color{gray}{(96$\times$2, 25, 75, 0, 100, 100, 0)} & 1.114 \color{gray}{(24$\times$1, 0, 100, 0, 100, 100, 0)} \\
\bottomrule
\end{tabular}
\end{table*}

%% file: tables/e2e_throughput_1_gpu_main_full.tex
\iftrue
\begin{table*}[t]
\centering
\vspace{-0.5em}
\caption{Generation throughput (token/s) on 1 GPU with different systems. Accelerate, DeepSpeed, and \sys use 1 GPU. Petals uses 1 GPU for OPT-6.7B, 4 GPUs for OPT-30B, and 24 GPUs for OPT-175B, but reports per-GPU throughput. Petals is benchmarked under different network delay and bandwidth. The models are run in INT8 as the default for Petals. We tune the batch size of each request to be 2 and issue requests by 6 parallel client processes to achieve the maximum throughput.
\sys is our system without compression; \sys (c) uses 4-bit compression. ``OOM'' means out-of-memory.}
\label{table:more_seq_len_1_gpu}
\begin{tabular}{llllllllll}
\toprule
Seq. length & \multicolumn{3}{c}{256} & \multicolumn{3}{c}{512} & \multicolumn{3}{c}{1024} \\
\cmidrule(lr){2-4}\cmidrule(lr){5-7}\cmidrule(lr){8-10}
Model size  & 6.7B   & 30B   & 175B  & 6.7B   & 30B   & 175B   & 6.7B     & 30B     & 175B     \\
 \midrule
Accelerate  & 50.66 & 1.34 & 0.02 & 25.12  & 0.62  & 0.01   & 13.01    & 0.31    & 0.01     \\
DeepSpeed   & 14.52 & 1.30 & 0.01 & 9.28   & 0.60  & 0.01   & 4.59     & 0.29    & OOM      \\
\midrule Petals ($<$5ms, 1Gb/s) & 9.03 & 3.55 & 0.09 & 8.25 & 2.84 & 0.08 & 6.56 & 1.51 & 0.06 \\
Petals ($<$5ms, 100Mb/s) & 9.15 & 2.53 & 0.06 & 8.18 & 1.67 & 0.05 & 6.52 & 0.87 & 0.03 \\
Petals (100ms, 100Mb/s) & 8.64 & 0.75 & 0.01 & 7.82 & 0.64 & 0.01 & 5.89 & 0.37 & 0.01 \\ \midrule
\sys        & 53.29 & 16.01 & 1.36 & 25.26  & 7.32  & 0.69   & 13.72    & 3.50    & 0.35     \\
\midrule
\sys (c)    & 56.72 & 16.86 & 2.26 & 29.12  & 8.70  & 1.12   & 13.18    & 3.98    & 0.42     \\
\bottomrule
\end{tabular}
\end{table*}
\fi

%% file: tables/e2e_throughput_1_gpu_policys.tex
\begin{table*}[t]
\centering
\vspace{-0.5em}
\caption{Generation throughput (token/s) on 1 GPU with \textbf{input sequence length 256 and output sequence length 32}. \sys is our system without compression; \sys (c) uses 4-bit compression. ``OOM'' means out-of-memory. The gray
tuple denotes a policy (GPU batch size $\times$ \#GPU-batch, $wg$, $wc$, $cg$, $cc$, $hg$, $hc$).}
\label{table:e2e_throughput_1_gpu_setup_256}
\footnotesize
\begin{tabular}{llll}
\toprule
Seq. length & \multicolumn{3}{c}{256} \\
\cmidrule(lr){2-4}
Model size  & 6.7B   & 30B   & 175B  \\
 \midrule
Accelerate  & 50.66 \color{gray}{(4$\times$1, 100, 0, 100, 0, 100, 0)}  & 1.34 \color{gray}{(16$\times$1, 0, 100, 100, 0, 100, 0)}  & 0.02 \color{gray}{(4$\times$1, 0, 0, 100, 0, 100, 0)}  \\
DeepSpeed   & 14.52 \color{gray}{(32$\times$1, 0, 100, 100, 0, 100, 0)} & 1.30 \color{gray}{(12$\times$1, 0, 100, 100, 0, 100, 0)}  &  0.01 \color{gray}{(2$\times$1, 0, 0, 100, 0, 100, 0)}  \\
\sys        & 53.29 \color{gray}{(4$\times$1, 100, 0, 100, 0, 100, 0)} & 16.01 \color{gray}{(160$\times$2, 10, 90, 0, 100, 0, 100)}  & 1.36 \color{gray}{(64$\times$8, 0, 50, 0, 0, 0, 100)}  \\
\midrule
\sys (c)    & 56.72 \color{gray}{(128$\times$1, 100, 0, 100, 0, 100, 0)} & 16.86 \color{gray}{(128$\times$8, 0, 100, 0, 100, 0, 100)}  & 2.26 \color{gray}{(96$\times$3, 0, 100, 0, 100, 0, 100)}  \\
\bottomrule
\end{tabular}
\end{table*}

\begin{table*}[t]
\centering
\caption{Generation throughput (token/s) on 1 T4 GPU with \textbf{input sequence length 512 and output sequence length 32}. \sys is our system without compression; \sys (c) uses 4-bit compression. ``OOM'' means out-of-memory. The gray
tuple denotes a policy (GPU batch size $\times$ \#GPU-batch, $wg$, $wc$, $cg$, $cc$, $hg$, $hc$).}
\label{table:e2e_throughput_1_gpu_setup_512}
\footnotesize
\begin{tabular}{llll}
\toprule
Seq. length & \multicolumn{3}{c}{512} \\
\cmidrule(lr){2-4}
Model size  & 6.7B   & 30B   & 175B  \\
 \midrule
Accelerate  & 25.12 \color{gray}{(2$\times$1, 100, 0, 100, 0, 100, 0)}  & 0.62 \color{gray}{(8$\times$1, 0, 100, 100, 0, 100, 0)}  & 0.01 \color{gray}{(2$\times$1, 0, 0, 100, 0, 100, 0)}  \\
DeepSpeed   & 9.28 \color{gray}{(16$\times$1, 0, 100, 100, 0, 100, 0)}   & 0.60 \color{gray}{(4$\times$1, 0, 100, 100, 0, 100, 0)}  & 0.01 \color{gray}{(1$\times$1, 0, 0, 100, 0, 100, 0)}  \\
\sys        & 25.26 \color{gray}{(2$\times$1, 100, 0, 100, 0, 100, 0)}  & 7.32 \color{gray}{(48$\times$3, 20, 80, 0, 100, 0, 100)}  & 0.69 \color{gray}{(32$\times$8, 0, 50, 0, 0, 0, 100)}  \\
\midrule
\sys (c)    & 29.12 \color{gray}{(72$\times$1, 100, 0, 100, 0, 100, 0)}  & 8.70 \color{gray}{(16$\times$20, 20, 80, 0, 100, 100, 0)}  & 1.12 \color{gray}{(48$\times$3, 0, 100, 0, 100, 0, 100)}  \\
\bottomrule
\end{tabular}
\end{table*}

\begin{table*}[t]
\centering
\vspace{-0.5em}
\caption{Generation throughput (token/s) on 1 T4 GPU with \textbf{input sequence length 1024 and output sequence length 32}. \sys is our system without compression; \sys (c) uses 4-bit compression. ``OOM'' means out-of-memory. The gray
tuple denotes a policy (GPU batch size $\times$ \#GPU-batch, $wg$, $wc$, $cg$, $cc$, $hg$, $hc$).}
\label{table:e2e_throughput_1_gpu_setup_1024}
\footnotesize
\begin{tabular}{lllllll}
\toprule
Seq. length & \multicolumn{3}{c}{1024} \\
\cmidrule(lr){2-4}\cmidrule(lr){5-7}
Model size  & 6.7B     & 30B     & 175B     \\
 \midrule
Accelerate  & 13.01 \color{gray}{(1$\times$1, 100, 0, 100, 0, 100, 0)}   & 0.31 \color{gray}{(4$\times$1, 0, 100, 100, 0, 100, 0)}    & 0.01 \color{gray}{(1$\times$1, 0, 0, 100, 0, 100, 0)}    \\
DeepSpeed   & 4.59 \color{gray}{(8$\times$1, 0, 100, 100, 0, 100, 0)}    & 0.29 \color{gray}{(2$\times$1, 0, 100, 100, 0, 100, 0)}    & OOM      \\
\sys        & 13.72 \color{gray}{(1$\times$1, 100, 0, 100, 0, 100, 0)}    & 3.50 \color{gray}{(20$\times$4, 4, 96, 0, 100, 0, 100)}    & 0.35 \color{gray}{(12$\times$12, 0, 50, 0, 0, 0, 100)}    \\
\midrule
\sys (c)    & 13.18 \color{gray}{(28$\times$1, 100, 0, 100, 0, 100, 0)}    & 3.98 \color{gray}{(20$\times$12, 0, 100, 0, 100, 0, 100)}   & 0.42 \color{gray}{(12$\times$4, 0, 100, 0, 100, 0, 100)}    \\
\bottomrule
\end{tabular}

\end{table*}

%% file: tables/e2e_throughput_1_gpu_setup_128_128.tex
\begin{table*}[t]
\centering
\caption{Generation throughput (token/s) on 1 T4 GPU with \textbf{input sequence length 128 and output sequence length 128}. \sys is our system without compression; \sys (c) uses 4-bit compression. ``OOM'' means out-of-memory. The gray
tuple denotes a policy (GPU batch size $\times$ \#GPU-batch, $wg$, $wc$, $cg$, $cc$, $hg$, $hc$).}
\label{table:e2e_throughput_1_gpu_setup_128_128}
\footnotesize
\begin{tabular}{llll}
\toprule
Seq. length & \multicolumn{3}{c}{128 + 128} \\
\cmidrule(lr){2-4}
Model size  & 6.7B   & 30B   & 175B  \\
 \midrule
Accelerate  & 73.411 \color{gray}{(5$\times$1, 100, 0, 100, 0, 100, 0)} & 1.547 \color{gray}{(16$\times$1, 0, 100, 100, 0, 100, 0)} & 0.021 \color{gray}{(4$\times$1, 0, 0, 100, 0, 100, 0)}\\
DeepSpeed   & 19.193 \color{gray}{(36$\times$1, 0, 100, 100, 0, 100, 0)} & 1.717 \color{gray}{(12$\times$1, 0, 100, 100, 0, 100, 0)} & 0.024 \color{gray}{(3$\times$1, 0, 0, 100, 0, 100, 0)} \\
\sys        & 106.404 \color{gray}{(7$\times$1, 100, 0, 100, 0, 100, 0)} & 24.634 \color{gray}{(32$\times$10, 25, 75, 0, 100, 100, 0)} & 2.409 \color{gray}{(64$\times$8, 0, 50, 0, 0, 0, 100)} \\
\midrule
\sys (c)    & 92.568 \color{gray}{(196$\times$1, 100, 0, 100, 0, 100, 0)} & 39.141 \color{gray}{(128$\times$8, 25, 75, 0, 100, 0, 100)} & 4.264 \color{gray}{(80$\times$3, 0, 100, 0, 100, 100, 0)}\\
\bottomrule
\end{tabular}
\end{table*}

%% file: tables/e2e_throughput_1_gpu_setup_512_8.tex
\begin{table*}[t]
\centering
\caption{Generation throughput (token/s) on 1 T4 GPU with \textbf{input sequence length 512 and output sequence length 8}. \sys is our system without compression; \sys (c) uses 4-bit compression. ``OOM'' means out-of-memory. The gray
tuple denotes a policy (GPU batch size $\times$ \#GPU-batch, $wg$, $wc$, $cg$, $cc$, $hg$, $hc$).}
\label{table:e2e_throughput_1_gpu_setup_512_8}
\footnotesize
\begin{tabular}{llll}
\toprule
Seq. length & \multicolumn{3}{c}{512 + 8} \\
\cmidrule(lr){2-4}
Model size  & 6.7B   & 30B   & 175B  \\
 \midrule
Accelerate  & 17.290 \color{gray}{(2$\times$1, 100, 0, 100, 0, 100, 0)} & 0.628 \color{gray}{(7$\times$1, 0, 100, 100, 0, 100, 0)} & 0.009 \color{gray}{(2$\times$1, 0, 0, 100, 0, 100, 0)}\\
DeepSpeed   & 9.055 \color{gray}{(18$\times$1, 0, 100, 100, 0, 100, 0)} & 0.872 \color{gray}{(6$\times$1, 0, 100, 100, 0, 100, 0)} & 0.007 \color{gray}{(1$\times$1, 0, 0, 100, 0, 100, 0)} \\
\sys        & 16.425 \color{gray}{(2$\times$1, 100, 0, 100, 0, 100, 0)} & 3.938 \color{gray}{(512$\times$8, 20, 80, 0, 100, 0, 100)} & 0.451 \color{gray}{(32$\times$8, 0, 50, 0, 0, 0, 100)} \\
\midrule
\sys (c)    & 14.244 \color{gray}{(76$\times$1, 100, 0, 100, 0, 100, 0)} & 4.019 \color{gray}{(16$\times$36, 25, 75, 0, 100, 0, 100)} & 0.559 \color{gray}{(48$\times$3, 0, 100, 0, 100, 0, 100)} \\
\bottomrule
\end{tabular}
\end{table*}

%% file: tables/pareto_full.tex
\begin{table*}[t]
\centering
\vspace{-0.5em}
\caption{The Pareto frontier of the latency-throughput trade-off of OPT-175B. The numbers are generation throughput (token/s) and effective batch latency (s) on 1 GPU with \textbf{input sequence length 512 and output sequence length 32}. The numbers in the parentheses are corresponding effective batch sizes. The numbers in bold are the best throughput and latency for each model. We organize the table so that the latency numbers of different methods in each row are similar for each model. The top value of each column corresponds to the setting of effective batch size 1.
(To reach the lowest latency, \sys uses an effective batch size of 2 rather than 1 because the latency difference between batch sizes 1 and 2 is 
 negligible in this case. So, a run with batch size 2 dominates the one with batch size 1 with higher throughput and similar latency.)}
\label{table:latency_throughput_175}
\begin{tabular}{rrrr}
\toprule
\multicolumn{4}{c}{175B (generation throughput / latency)}\\
\cmidrule(rrrr){1-4}
Accelerate & DeepSpeed & FlexGen & FlexGen (c) \\
\midrule
-  &  - & - & 0.052 / \textbf{612} (1)\\
-  &  - & - & 0.198 / 647 (4)\\
-  &  - & - & 0.369 / 693 (8)\\
-  &  - & - & 0.779 / 1973 (48)\\
-  &  - & 0.025 / 2555 (2) & 1.092 / 2813 (96)\\
-  &  - & 0.254 / 4028 (32) & \textbf{1.122} / 4072 (144)\\
-  & 0.006 / 5024 (1) & 0.421 / 4864 (64) & - \\
-  &  - & 0.572 / 7159 (128) & - \\ 
 0.004 / 7508 (1) & - & - & - \\
0.008 / 7633 (2) & - & - & - \\
- &  - & 0.687 / 11916 (256) & - \\
\bottomrule
\end{tabular}
\vspace{-0.5em}
\end{table*}

\begin{table*}[t]
\centering
\vspace{-0.5em}
\caption{The Pareto frontier of the latency-throughput trade-off of OPT-30B. The numbers are generation throughput (token/s) and effective batch latency (s) on 1 GPU with \textbf{input sequence length 512 and output sequence length 32}. The numbers in the parentheses are corresponding effective batch sizes. The numbers in bold are the best throughput and latency for each model. We organize the table so that the latency numbers of different methods in each row are similar for each model. The top value of each column corresponds to the setting of effective batch size 1.
}
\label{table:latency_throughput_30}
\begin{tabular}{rrrr}
\toprule
\multicolumn{4}{c}{30B (generation throughput / latency)}\\
\cmidrule(rrrr){1-4}
Accelerate & DeepSpeed & FlexGen & FlexGen (c) \\
\midrule
 - & - & - & 0.21 / \textbf{153} (1) \\
 - & - & - & 0.42 / 154 (2) \\
 - & - & 0.20 / 159 (1) & 0.82 / 155 (4) \\
 - & - & 0.37 / 172 (2) & 1.58 / 162 (8) \\
 - & - & 0.73 / 174 (4) & 2.88 / 178 (16) \\
 - & 0.16 / 203 (1) & 1.40 / 183 (8) & - \\
 - & 0.31 / 204 (2) & 2.70 / 190 (16) & - \\
 - & 0.62 / 206 (4) & 4.05 / 253 (32) & 4.63 / 277 (40) \\ 
 0.08 / 405 (1) & - & 5.71 / 359 (64) & 6.72 / 381 (80) \\
 0.31 / 408 (4) & - & - & - \\
 0.62 / 413 (8) & - & - & - \\
 - & - & 7.32 / 559 (144) & - \\
 - & - & - & 7.96 / 644 (160) \\
 - & - & - & 8.49 / 904 (240) \\
 - & - & - & \textbf{8.70} / 1177 (320) \\
\bottomrule
\end{tabular}
\vspace{-0.5em}
\end{table*}

%% file: tables/ablation_full.tex
\begin{table*}[t]
\centering
\caption{Ablation study of policies. The numbers correspond to generation \textbf{throughput} on 1 GPU with \textbf{input sequence length 512 and output sequence length 32}.
All policies have CPU computation turned on.
The numbers for OPT-175B show some inconsistency with the end-to-end evaluation in \cref{table:e2e_throughput_1_gpu} and \cref{table:e2e_throughput_1_gpu_setup_512} (0.49 vs 0.69) because we turn on the pagecache-mangagement~\cite{pagecache-mangagement} tool to prevent the automatic disk cache in operating systems, which makes the ablation results more accurate but brings some overheads. This added some overhead and misses the advantage of using CPU cache. A real run should be expected to have a better throughput.
($gbs$ denotes the GPU batch size, $\#gb$ denotes the number of GPU batches in a block.)}
\label{table:tech_ablation_throughput}
\begin{tabular}{rrrrrrrrrr}
\toprule
$gbs$ & $\#gb$ & $wg$ & $wc$ & $cg$ & $cc$ & $hg$ & $hc$ & 30B (token/s)   & 175B (token/s) \\
\midrule
48 & 3 & 20 & 80 & 0 & 100 & 0 & 100 & \textbf{7.32} & OOM \\
48 & 3 & 0 & 100 & 0 & 100 & 0 & 100 & 7.26 & OOM \\
48 & 1 & 20 & 80 & 0 & 100 & 0 & 100 & 5.40 & OOM \\
32 & 8 & 0 & 50 & 0 & 0 & 0 & 100 & 1.66 & \textbf{0.49} \\
32 & 8 & 0 & 0 & 0 & 0 & 0 & 100 & 1.55 & 0.44 \\
32 & 1 & 0 & 50 & 0 & 0 & 0 & 100 & 0.88 & 0.23 \\
1 & 1 & 20 & 80 & 100 & 0 & 100 & 0 & 0.20 & OOM \\
1 & 1 & 0 & 50 & 100 & 0 & 100 & 0 & 0.04 & 0.01 \\
8 & 1 & 0 & 100 & 100 & 0 & 100 & 0 & 1.57 & OOM \\
2 & 1 & 0 & 0 & 100 & 0 & 100 & 0 & 0.05 & 0.01 \\
\bottomrule
\end{tabular}
\end{table*}

\begin{table*}[t]
\centering
\caption{Ablation study of policies. The numbers are full generation \textbf{latency} on 1 GPU with \textbf{input sequence length 512 and output sequence length 32}. 
All policies have CPU computation turned on.
We turn on the pagecache-mangagement~\cite{pagecache-mangagement} tool to prevent the automatic disk cache in operating systems, which makes the ablation results more accurate but brings some overheads. This added some overhead and misses the advantage of using CPU cache. A real run should be expected to have a better latency.
($gbs$ denotes the GPU batch size, $\#gb$ denotes the number of GPU batches in a block.)}
\label{table:tech_ablation_latency}
\begin{tabular}{rrrrrrrrrr}
\toprule
$gbs$ & $\#gb$ & $wg$ & $wc$ & $cg$ & $cc$ & $hg$ & $hc$ & 30B (s)   & 175B (s) \\
\midrule
48 & 3 & 20 & 80 & 0 & 100 & 0 & 100 & 559 & OOM \\
48 & 3 & 0 & 100 & 0 & 100 & 0 & 100 & 635 & OOM \\
48 & 1 & 20 & 80 & 0 & 100 & 0 & 100 & 284 & OOM \\
32 & 8 & 0 & 50 & 0 & 0 & 0 & 100 & 4930 & 16611 \\
32 & 8 & 0 & 0 & 0 & 0 & 0 & 100 & 5287 & 18704 \\
32 & 1 & 0 & 50 & 0 & 0 & 0 & 100 & 1164 & 4476 \\
1 & 1 & 20 & 80 & 100 & 0 & 100 & 0 & \textbf{160} & OOM \\
1 & 1 & 0 & 50 & 100 & 0 & 100 & 0 & 737 & \textbf{3107} \\
8 & 1 & 0 & 100 & 100 & 0 & 100 & 0 & 170 & OOM \\
2 & 1 & 0 & 0 & 100 & 0 & 100 & 0 & 1215 & 6072 \\
\bottomrule
\end{tabular}
\end{table*}

%% file: tables/main_ablation_policy.tex
\begin{table*}[t]
\centering
\vspace{-0.5em}
\caption{Ablation study of proposed techniques. The numbers are generation throughput on 1 T4 GPU with prompt length 512 and generating length 32. The gray tuple denotes a policy (GPU batch size $\times$ \#GPU-batch, $wg$, $wc$, $cg$, $cc$, $hg$, $hc$).}
\label{table:main_ablation_policy}
\begin{tabular}{l|lll}
\toprule
Model size          & 30B     & 175B    \\
\midrule
All optimizations   & 7.32 \color{gray}{(48$\times$3, 20, 80, 0, 100, 0, 100)} & 0.69 \color{gray}{(32$\times$8, 0, 50, 0, 0, 0, 100)} \\
No policy search    & 7.26 \color{gray}{(48$\times$3, 0, 100, 0, 100, 0, 100)} &  0.27 \color{gray}{(32$\times$1, 0, 50, 0, 0, 0, 100)}  \\

No overlapping      & 5.86 \color{gray}{(48$\times$3, 20, 80, 0, 100, 0, 100)} & 0.59 \color{gray}{(32$\times$8, 0, 50, 0, 0, 0, 100)}     \\
No CPU compute      & 4.03 \color{gray}{(48$\times$3, 20, 80, 0, 100, 0, 100)}  & 0.62 \color{gray}{(32$\times$8, 0, 50, 0, 0, 0, 100)}     \\
No disk             & 7.32 \color{gray}{(48$\times$3, 20, 80, 0, 100, 0, 100)}  & OOM      \\
w/ DeepSpeed policy & 1.57 \color{gray}{(8$\times$1, 0, 100, 100, 0, 100, 0)}   & 0.01 \color{gray}{(2$\times$1, 0, 0, 100, 0, 100, 0)}  \\
\bottomrule
\end{tabular}
\vspace{-1em}
\end{table*}

%% file: tables/ssd_throughput.tex
\begin{table*}[t]
\centering
\caption{Generation throughput (token/s) on hardware specifed in \cref{table:hardware_specs} with \textbf{input sequence length 512 and output sequence length 32}.
The performance of OPT-30B is not affected because OPT-30B does not use SSD. The disk speed is measured using the Linux command dd with a block size (bs) of 1MB and the number of blocks (count) of 16000. The PageCacheManagement tool is used to disable disk cache in the operating system during measurement.}
\label{table:ssd_throughput}
\begin{tabular}{lrr}
\toprule
Disk Specification &30B   & 175B  \\
 \midrule
1.6GB/s read, 1.3GB/s write (local SSD, the one used in the main paper) & 7.32 & 0.69 \\
0.5GB/s read, 0.5GB/s write (persistent SSD, a new setting) & 7.32 & 0.30 \\
1.6GB/s read, 1.3GB/s write (local SSD, use PageCacheManagement) & 7.32 & 0.49 \\
0.5GB/s read, 0.5GB/s write (persistent SSD, use PageCacheManagement) & 7.32 & 0.292 \\
\bottomrule
\end{tabular}
\end{table*}

%% file: tables/helm_mixed_length.tex
\begin{table*}[t]
\centering
\caption{Selected example of FlexGen on real-world tasks from the HELM benchmark, which consists of prompts of various lengths with different output lengths.
We use two metrics: padded throughput = (number of tokens in padded prompts + number of tokens in padded outputs) / latency, actual throughput = (number of non-padding tokens in prompts + number of non-padding tokens in outputs) / latency.
The throughput are measured in token/s.
To batch sequences of variable lengths, FlexGen simply pads all inputs to the maximum prompt length, which is a common method used in many systems. Depending on the distribution of the prompt length, the efficiency of this simple padding method varies. For example, if most sequences have similar lengths, then the batching efficiency should be very high. if some sequences are very long and some sequences are short, then FlexGen will spend a lot of time on the useless computation of padding tokens.}
\label{table:helm_mixed_length}
\footnotesize
\begin{tabular}{p{7em}rrrrr}
\toprule
Task & Padded input seq. length & Padded output seq. length & Padded throughput & Actual throughput & Efficiency \\
\midrule
MMLU\\ (abstract\_algebra) & 512 & 1 & 251.5 & 188.6 & 75.0\% \\
\midrule
xsum & 1984 & 64 & 60.5 & 47.6 & 78.7\% \\
\bottomrule
\end{tabular}
\end{table*}